\documentclass{article}

\PassOptionsToPackage{numbers, compress}{natbib}

\usepackage[final]{neurips_2021}

\usepackage[utf8]{inputenc} 
\usepackage[T1]{fontenc}    
\usepackage[colorlinks,
            linkcolor=red,
            anchorcolor=blue,
            citecolor=blue
            ]{hyperref}
\usepackage{url}            
\usepackage{booktabs}       
\usepackage{amsfonts}       
\usepackage{nicefrac}       
\usepackage{microtype}      
\usepackage{xcolor}         
\usepackage{smile}
\newcommand{\state}{s}
\usepackage{enumitem}
\usepackage{graphicx}
\usepackage{subfigure}
\usepackage{pgfplots}
\usetikzlibrary{arrows,shapes,snakes,automata,backgrounds,petri}
\usepackage{booktabs} 

\allowdisplaybreaks
\newcommand{\la}{\langle}
\newcommand{\ra}{\rangle}
\newcommand{\qvalue}{Q}
\newcommand{\vvalue}{V}

\newcommand{\reward}{r}

\usepackage{colortbl}
\definecolor{LightCyan}{rgb}{0.8, 0.9, 1}
\def \CCC {}
\def \algname {\text{\text{UCBVI-}$\gamma$}}
\usepackage{enumitem}
\usepackage[textsize=tiny]{todonotes}
\usepackage[compact]{titlesec}
\usepackage{sidecap}
\titlespacing*{\section}{0pt}{*0.1}{*0.1}
\titlespacing*{\subsection}{0pt}{*0.1}{*0.1}
\titlespacing*{\subsubsection}{0pt}{*0.1}{*0.1}

\title{\huge Nearly Minimax Optimal Reinforcement Learning for Discounted MDPs}

\author{%
  Jiafan He \\
  Department of Computer Science\\
 University of California, Los Angeles\\
  CA 90095, USA \\
  \texttt{jiafanhe19@ucla.edu} \\
   \And
    Dongruo Zhou \\
  Department of Computer Science\\
 University of California, Los Angeles\\
  CA 90095, USA \\
  \texttt{drzhou@cs.ucla.edu} \\
     \And
    	Quanquan Gu \\
  Department of Computer Science\\
 University of California, Los Angeles\\
  CA 90095, USA \\
  \texttt{ qgu@cs.ucla.edu} \\
}

\date{}

\begin{document}
\maketitle

%
\begin{abstract}
 We study the reinforcement learning problem for discounted Markov Decision Processes (MDPs) under the tabular setting. We propose a model-based algorithm named UCBVI-$\gamma$, which is based on the \emph{optimism in the face of uncertainty principle} and the Bernstein-type bonus. We show that UCBVI-$\gamma$ achieves an $\tilde{O}\big({\sqrt{SAT}}/{(1-\gamma)^{1.5}}\big)$ regret, where $S$ is the number of states, $A$ is the number of actions, $\gamma$ is the discount factor and $T$ is the number of steps. In addition,  we construct a class of hard MDPs and show that for any algorithm, the expected regret is at least $\tilde{\Omega}\big({\sqrt{SAT}}/{(1-\gamma)^{1.5}}\big)$. Our upper bound matches the minimax lower bound up to logarithmic factors, which suggests that UCBVI-$\gamma$ is nearly minimax optimal for discounted MDPs.
\end{abstract}

\section{Introduction}

The goal of reinforcement learning (RL) is designing algorithms to learn the optimal policy through interactions with the unknown dynamic environment. Markov decision process (MDPs) plays a central role in reinforcement learning due to their ability to describe the time-independent state transition property. More specifically, the discounted MDP is one of the standard MDPs in reinforcement learning to describe sequential tasks without interruption or restart. For discounted MDPs, with a \emph{generative model} \citep{kakade2003sample}, several algorithms with near-optimal sample complexity have been proposed. More specifically, \citet{azar2013minimax} proposed an Empirical QVI algorithm which achieves the optimal sample complexity to find the optimal value function. \citet{sidford2018near} proposed a sublinear randomized value iteration algorithm that achieves a near-optimal sample complexity to find the optimal policy, and \citet{sidford2018variance} further improved it to reach the optimal sample complexity. Since generative model is a powerful oracle that allows the algorithm to query the reward function and the next state for any state-action pair $(s, a)$, it is natural to ask whether there exist online RL algorithms (without generative model) that achieve optimality. 

To measure an online RL algorithm, a widely used notion is \emph{regret}, which is defined as the summation of sub-optimality gaps over time steps. The regret is firstly introduced for episodic and infinite-horizon average-reward MDPs and later extended to discounted MDPs by \citep{liu2020regret,yang2020q,zhou2020provably,zhou2020provably}.
\citet{liu2020regret} proposed a double Q-learning algorithm with the UCB exploration (Double Q-learning), which enjoys $\tilde O(\sqrt{SAT}/(1-\gamma)^{2.5})$ regret, where $S$ is the number of states, $A$ is the number of actions, $\gamma$ is the discount factor and $T$ is the number of steps. While Double Q-learning enjoys a standard $\sqrt{T}$-regret, it still does not match the lower bound 
proved in \cite{liu2020regret} in terms of the dependence on $S,A$ and $1/(1-\gamma)$. Recently, \citet{zhou2020nearly} proposed a $\text{UCLK}^+$ algorithm for  discounted MDPs under the linear mixture MDP assumption and achieved $\tilde O\big(d\sqrt{T}/(1-\gamma)^{1.5}\big)$ regret, where $d$ is the dimension of the feature mapping. However, directly applying their algorithm to our setting would yield an $\tilde O\big(S^2A\sqrt{T}/(1-\gamma)^{1.5}\big)$ regret\footnote{Linear mixture MDP assumes that there exists a feature mapping $\bphi(s'|s,a) \in \RR^d$ and a vector $\btheta \in \RR^d$ such that $\PP(s'|s,a) = \la \bphi(s'|s,a), \btheta\ra$. It can be verified that any MDP is automatically a linear mixture MDP with a $S^A$-dimensional feature mapping \citep{ayoub2020model, zhou2020provably}.}, 
which is even worse that of double Q-learning \citep{liu2020regret} in terms of the dependence on $S,A$.

 
In this paper, we aim to close this gap by designing a practical algorithm with a nearly optimal regret. In particular, we propose a model-based algorithm named $\algname$ for discounted MDPs without using the generative model. At the core of our algorithm is to use a ``refined'' Bernstein-type bonus and the \emph{law of total variance} \citep{azar2013minimax, azar2017minimax}, which together can provide tighter upper confidence bound (UCB). 
Our contributions are summarized as follows:
\begin{itemize}[leftmargin = *]
    \item We propose a model-based algorithm $\algname$ to learn the optimal value function under the discounted MDP setting. We show that the regret of $\algname$ in first $T$ steps is upper bounded by $\tilde O({\sqrt{SAT}}/{(1-\gamma)^{1.5}})$. Our regret bound strictly improves the best existing regret $\tilde O({\sqrt{SAT}}/{(1-\gamma)^{2.5}})$\footnote{\label{note1}The regret definition in \cite{liu2020regret} differs from our definition by a factor of $(1-\gamma)^{-1}$. Here we translate their regret from their definition to our definition for a fair comparison. A detailed comparison can be found in Appendix.} in \cite{liu2020regret} by a factor of $(1-\gamma)^{-1}$. 
    \item We also prove a lower bound of the regret by constructing a class of hard-to-learn discounted MDPs, which can be regarded as a \emph{chain} of the hard MDPs considered in \cite{liu2020regret}. We show that for any algorithm, its regret in the first $T$ steps can not be lower than $\tilde \Omega({\sqrt{SAT}}/{(1-\gamma)^{1.5}})$ on the constructed MDP. This lower bound also strictly improves the lower bound  $\Omega(\sqrt{SAT}/(1-\gamma) + \sqrt{AT}/(1-\gamma)^{1.5})$ proved by \cite{liu2020regret}.
    \item The nearly matching upper and the lower bounds together suggest that the proposed $\algname$ algorithm is minimax-optimal up to logarithmic factors. 
\end{itemize}

    
We compare the regret of $\algname$ with previous online algorithms for learning discounted MDPs in Table \ref{table:11}. 

\noindent\textbf{Notation} 
For any positive integer $n$, we denote by $[n]$ the set $\{1,\dots,n\}$. For any two numbers $a$ and $b$, we denote by $a \vee b$ as the shorthand for $\max(a,b)$. For two sequences $\{a_n\}$ and $\{b_n\}$, we write $a_n=O(b_n)$ if there exists an absolute constant $C$ such that $a_n\leq Cb_n$, and we write $a_n=\Omega(b_n)$ if there exists an absolute constant $C$ such that $a_n\geq Cb_n$. We use $\tilde O(\cdot)$ and $\tilde \Omega(\cdot)$ to further hide the logarithmic factors. 

\section{Related Work}\label{section: related}

\newcolumntype{g}{>{\columncolor{LightCyan}}c}
\begin{table*}[ht]
\caption{Comparison of RL algorithms for discounted MDPs in terms of sample complexity and regret. Note that the regret bounds for all the compared algorithms except Double Q-learning \citep{liu2020regret} are derived from their sample complexity results. See Appendix~\ref{app:conversion} for more details.}\label{table:11}
\centering
\begin{tabular}{cggg}
\toprule
\rowcolor{white}
& Algorithm & Sample complexity & Regret
 \\
\midrule
\rowcolor{white}
&Delay-Q-learning  &  & \\
\rowcolor{white}
 &\small{\citep{strehl2006pac}} & \multirow{-2}{*}{$\tilde  O\Big(\frac{SA}{(1-\gamma)^8\epsilon^4}\Big)$}  &  \multirow{-2}{*}{$\tilde  O\Big(\frac{S^{1/5}A^{1/5}T^{4/5}}{(1-\gamma)^{9/5}}\Big)$}\\

\rowcolor{white}
 &Q-learning with UCB &   &  \\
 \rowcolor{white}
 &\small{\citep{dong2019q}} &\multirow{-2}{*}{ $ \tilde O\Big(\frac{SA}{(1-\gamma)^7\epsilon^2}\Big)$} & \multirow{-2}{*}{$\tilde  O\Big(\frac{S^{1/3}A^{1/3}T^{2/3}}{(1-\gamma)^{8/3}}\Big)$} \\

 \rowcolor{white}
  &UCB-multistage &   &  \\
  \rowcolor{white}
 &\small{\citep{zhang2020model}} & \multirow{-2}{*}{ $ \tilde O\Big(\frac{SA}{(1-\gamma)^{5.5}\epsilon^2}\Big)$}& \multirow{-2}{*}{$\tilde  O\Big(\frac{S^{1/3}A^{1/3}T^{2/3}}{(1-\gamma)^{13/6}}\Big)$}\\
 \rowcolor{white}
   &{ UCB-multistage-adv } &   &  \\
   \rowcolor{white}
 &\small{\citep{zhang2020model}} &\multirow{-2}{*}{ $\tilde O\Big(\frac{SA}{(1-\gamma)^3\epsilon^2}\Big)$\footnotemark} & \multirow{-2}{*}{$\tilde  O\Big(\frac{S^{1/3}A^{1/3}T^{2/3}}{(1-\gamma)^{4/3}}\Big)$}\\
 
 \rowcolor{white}
 &Double Q-learning &   &  \\
 \rowcolor{white}
 \multirow{-9}{*}{Model-free}&\citep{liu2020regret} &\multirow{-2}{*}{N/A} &\multirow{-2}{*}{$\tilde O\Big(\frac{\sqrt{SAT}}{(1-\gamma)^{2.5}}\Big)$}\\
 
 \midrule
 \rowcolor{white}
   &R-max  &   &  \\
\rowcolor{white}
 &\small{\citep{brafman2002r}} & \multirow{-2}{*}{$\tilde O\Big(\frac{S^2A}{(1-\gamma)^6\epsilon^3}\Big)$} & \multirow{-2}{*}{$\tilde  O\Big(\frac{S^{1/2}A^{1/4}T^{3/4}}{(1-\gamma)^{7/4}}\Big)$}\\
 
 \rowcolor{white}
 &MoRmax  &   &  \\
 \rowcolor{white}
 &\small{\citep{szita2010model}} & \multirow{-2}{*}{$\tilde O\Big(\frac{SA}{(1-\gamma)^6\epsilon^2}\Big)$}& \multirow{-2}{*}{$\tilde  O\Big(\frac{S^{1/3}A^{1/3}T^{2/3}}{(1-\gamma)^{7/3}}\Big)$}\\

\rowcolor{white}
 & UCRL  &   &  \\
 \rowcolor{white}
&\small{\citep{lattimore2012pac}} & \multirow{-2}{*}{$\tilde O\Big(\frac{S^2A}{(1-\gamma)^3\epsilon^2}\Big)$} & \multirow{-2}{*}{$\tilde  O\Big(\frac{S^{2/3}A^{1/3}T^{2/3}}{(1-\gamma)^{4/3}}\Big)$}\\

 &$\algname$  &   &  \\
 
\multirow{-8}{*}{Model-based} &(\textbf{Our work})  &\multirow{-2}{*}{ N/A} &\multirow{-2}{*}{$\tilde O\Big(\frac{\sqrt{SAT}}{(1-\gamma)^{1.5}}\Big)$}\\
 \midrule
 \multirow{3}{*}{Lower bound} &  &   &  \\
 & &\multirow{-2}{*}{$\tilde \Omega\Big(\frac{SA}{(1-\gamma)^3\epsilon^2}\Big)$} & \multirow{-2}{*}{$\tilde \Omega\Big(\frac{\sqrt{SAT}}{(1-\gamma)^{1.5}}\Big)$}\\
&  \multirow{-3}{*}{ N/A}&  \small{\citep{lattimore2012pac}} &(\textbf{Our work})\\
\bottomrule
\end{tabular}

\footnotesize{\noindent 2. It holds when $\epsilon \leq 1/\text{poly}(S,A,1/(1-\gamma))$.}
\end{table*}

\noindent\textbf{Model-free Algorithms for Discounted MDPs.} A large amount of reinforcement learning algorithms like Q-learning can be regarded as model-free algorithms. These algorithms directly learn the action-value function by updating the values of each state-action pair. \citet{kearns1999finite} firstly proposed a phased Q-Learning which learns an $\epsilon$-optimal policy with $\tilde O({SA}/{((1-\gamma)^7\epsilon^2}))$ sample complexity for $\epsilon\leq 1/(1-\gamma)$. Later on, \citet{strehl2006pac} proposed a delay-Q-learning algorithm, which achieves $\tilde  O({SA}/{((1-\gamma)^8\epsilon^4}))$ sample complexity of exploration. \citet{wang2017randomized} proposed 
a randomized primal-dual method algorithm, which improves the sample complexity to $\tilde O({SA}/{((1-\gamma)^4\epsilon^2}))$  for $\epsilon\leq 1/(1-\gamma)$ under the ergodicity assumption. Later, \citet{sidford2018variance} proposed a sublinear randomized value iteration algorithm and achieved $\tilde O({SA}/{((1-\gamma)^4\epsilon^2}))$ sample complexity for $\epsilon\leq 1$. \citet{sidford2018near} further improved the empirical QVI algorithm and proposed a variance-reduced QVI algorithm, which improves the sample complexity to $\tilde O({SA}/{((1-\gamma)^3\epsilon^2}))$ for $\epsilon\leq 1$. \citet{wainwright2019variance} proposed a variance-reduced Q-learning algorithm, which is an extension of
the Q-learning algorithm and achieves $\tilde O({SA}/{((1-\gamma)^3\epsilon^2}))$ sample complexity. 
In addition, \citet{dong2019q} proposed an infinite Q-learning with UCB and improved the sample complexity of exploration to  $\tilde O({SA}/{((1-\gamma)^7\epsilon^2}))$. \citet{zhang2020model} proposed a UCB-multistage algorithm which attains the  $\tilde O({SA}/{((1-\gamma)^{5.5}\epsilon^2}))$ sample complexity of exploration, and proposed a UCB-multistage-adv algorithm which attains a better sample complexity $\tilde O({SA}/{((1-\gamma)^{3}\epsilon^2}))$ in the high accuracy regime. 
Recently, \citet{liu2020regret} focused on regret minimization for the infinite-horizon discounted MDP and showed the connection between regret and sample complexity of exploration. \citet{liu2020regret} proposed a Double Q-Learning algorithm, which achieves  $\tilde O({\sqrt{SAT}}/{(1-\gamma)^{2.5}})$ regret within $T$ steps. Furthermore, \citet{liu2020regret} constructed a series of hard MDPs and showed that the expected regret for any algorithm is lower bounder by $ \tilde{\Omega}\big(\sqrt{SAT}/(1-\gamma)+{\sqrt{AT}}/{(1-\gamma)^{1.5}}\big)$. There still exists a  ${1}/{(1-\gamma)}$-gap between the upper and lower regret bounds. In contrast to the aforementioned model-free algorithms, our proposed algorithm is model-based.

\noindent\textbf{Model-based Algorithms for Discounted MDP.} Our $\algname$ falls into the category of model-based reinforcement learning algorithms. Model-based algorithms maintain a model of the environment and update it based on the observed data. They will form the policy based on the learnt model. More specifically, to learn the $\epsilon$-optimal value function, \citet{azar2013minimax} proposed an empirical QVI algorithm which achieves $\tilde O({SA}/{((1-\gamma)^3\epsilon^2}))$ sample complexity. \citet{azar2013minimax} proposed an empirical QVI algorithm which improves the sample complexity to $\tilde O({SA}/{((1-\gamma)^3\epsilon^2}))$ for $\epsilon\leq 1/\sqrt{(1-\gamma)S}$. \citet{szita2010model} proposed an MoRmax algorithm, which achieves $\tilde O({SA}/{((1-\gamma)^6\epsilon^2}))$ sample complexity.  Later, \citet{lattimore2012pac} proposed a UCRL algorithm, which achieves $\tilde O({S^2A}/{((1-\gamma)^3\epsilon^2}))$ sample complexity in general and $\tilde O({SA}/{((1-\gamma)^3\epsilon^2}))$ sample complexity with a strong assumption on the state transition. Recently, \citet{agarwal2019model} proposed a refined analysis for the empirical QVI algorithm which achieves $\tilde O({SA}/{((1-\gamma)^3\epsilon^2}))$ sample complexity when $\epsilon\leq 1/\sqrt{1-\gamma}$.

\noindent\textbf{Upper and Lower Bounds for Episodic MDPs.}
There is a line of work which aims at proving sample complexity or regret for episodic MDPs (MDPs which consist of restarting episodes) \citep{dann2015sample, osband2016lower,azar2017minimax, osband2017posterior,jin2018q, dann2019policy, simchowitz2019non, russo2019worst, zanette2019tighter,zhang2020almost,neu2020unifying, pacchiano2020optimism}. 
Compared with the episodic MDP, discounted MDPs involve only one infinite-horizon sample trajectory, suggesting that any two states or actions on the trajectory are dependent. Such a dependence makes the learning of discounted MDPs more challenging.

\section{Preliminaries}\label{section: pre}

We consider infinite-horizon discounted Markov Decision Processes (MDP) which are defined by a tuple $(\cS,\cA,\gamma,\reward,\PP)$. Here $\cS$ is the state space with $|\cS| = S$, $\cA$ is the action space with $|\cA| = A$, $\gamma \in (0,1)$ is the discount factor, $\reward: \cS \times \cA \rightarrow [0,1]$ is the reward function, $\PP(s'|s,a)$ is the transition probability function, which denotes the probability that state $s$ transfers to state $s'$ with action $a$. For simplicity, we assume the reward function is \emph{deterministic and known}.  
A \emph{non-stationary policies} $\pi$ is a collection of function $\{\pi_t\}_{t=1}^{\infty}$, where each function $\pi_t:\{\cS\times \cA\}^{t-1}\times \cS \rightarrow \cA$ maps history $\{s_1,a_1,...,s_{t-1},a_{t-1},s_t=s\}$ to an action. For any non-stationary policy $\pi$, we denote $\pi_{t}(s)=\pi_t(s;s_1,a_1,...,s_{t-1},a_{t-1})$ for simplicity. We define the action-value function and value function at step $t$ as follows: 
\begin{align}
    &\qvalue^{\pi}_t(s,a) =\EE\bigg[\sum_{i=0}^\infty \gamma^{i}\reward(s_{t+i},a_{t+i})\bigg|s_1,...,s_t=s,a_t=a\bigg],\notag\\
    & \vvalue^{\pi}_t(s) = \EE\bigg[\sum_{i=0}^\infty \gamma^{i}\reward(s_{t+i}, a_{t+i})\bigg|s_1,...,s_t=s\bigg],\notag
\end{align}
where $a_{t+i}=\pi_{t+i}(s_{t+i})$, and $s_{t+i+1}\sim  \PP\big(\cdot|s_{t+i}, \pi_{t+i}(s_{t+i})\big)$.
In addition, we denote the optimal action-value function and the optimal value function as $\qvalue^*(s, a) = \sup_{\pi}\qvalue^{\pi}_1(s,a)$ and $\vvalue^*(s) = \sup_{\pi}\vvalue^{\pi}_1(s)$ respectively. Note that the optimal action-value function and the optimal value function are independent of the step $t$. For simplicity, for any function $V:\cS\rightarrow R$, we denote $[\PP V](s,a)=\EE_{s'\sim \PP(\cdot |s,a)}V(s')$. According to the definition of the value function, we have the following non-stationary Bellman equation and Bellman optimality equation for non-stationary policy $\pi$ and optimal policy ${\pi^*}$: 
\begin{align}
    \qvalue^{\pi}_t(s,a) &= \reward(s,a) + \gamma [\PP \vvalue^{\pi}_{t+1}](s,a),\ \qvalue^*(s,a)
    = \reward(s,a) + \gamma [\PP \vvalue^*](s,a).\label{eq:bellman}
\end{align}
\section{Main Results}\label{section: main theorem}
\subsection{Algorithm}
In this subsection, we propose the Upper Confidence Bound Value Iteration-$\gamma$ (UCBVI-$\gamma$) algorithm, which is illustrated in Algorithm \ref{algorithm}. The algorithm framework of $\algname$ follows the UCBVI algorithm proposed in \citet{azar2017minimax}, which can be regarded as the counterpart of $\algname$ in the episodic MDP setting. 


$\algname$ is a model-based algorithm that maintains an empirical measure $\PP_t$ at each step $t$. At the beginning of the $t$-th iteration, $\algname$ takes action $a_t$ based on the greedy policy induced by $Q_t(s_t,a)$ and transits to the next state $s_{t+1}$. After receiving the next state $s_{t+1}$, $\algname$ computes the empirical transition probability function $\PP_t(s'|s,a)$ in~\eqref{eq:empirical_tran}.
Based on empirical transition probability function $\PP_t(s'|s,a)$, $\algname$ updates $Q_{t+1}(s,a)$ by performing one-step value iteration on $\qvalue_t(s,a)$ with an additional upper confidence bound (UCB) term $\text{UCB}_t(s,a)$ defined in \eqref{eq:UCB1}. Here the UCB bonus term is used to measure the uncertainty of the expectation of the value function $V_{t}(s)$. Unlike previous work, which adapts a Hoeffding-type bonus \citep{liu2020regret}, our $\algname$ uses a Bernstein-type bonus which brings a tighter upper bound by accessing the variance of $\vvalue_{t}(s)$, denoted by $\text{Var}_{s' \sim \PP(\cdot|, s,a)}V_t(s')$.
However, since the probability transition $\PP(\cdot|s,a)$ is unknown, it is impossible to calculate the exact variance of $V_t$. Instead, $\algname$ estimates the variance by considering the variance of $V_t$ over the empirical probability transition function $\PP_t(\cdot|s,a)$ defined in \eqref{eq:empirical_tran}. Therefore, the final UCB bonus term in~\eqref{eq:UCB1} can be regarded as a standard Bernstein-type bonus on the empirical measure $\PP_t(\cdot|s,a)$ with an additional error term.


Compared with UCBVI algorithm in \citet{azar2017minimax}, the action-value function $Q_t(s,a)$ in $\algname$ is updated in a forward way from step $1$ to step $T$ with the initial value $\qvalue_1(s,a)=1/(1-\gamma)$ for all $s\in \cS,a\in \cA$, while UCBVI updates its action-value function in a backward way from $\qvalue_{t,H}$ to $\qvalue_{t,1}$ with initial value $\qvalue_{t,H}(s,a)=0$. Compared with UCRL in \citet{lattimore2012pac}, $\algname$ does not need to call an additional extended value iteration sub-procedure \citep{jaksch2010near, strehl2008analysis}, which is not easy to implement even with infinite computation \citep{lattimore2012pac}. 

\paragraph{Computational complexity}
\CCC{In each step $t$, Algorithm \ref{algorithm} needs to first compute the empirical transition $\PP_{t}$ and update the value function $V_{t+1}$ by one-step value iteration, which will cost $O(S^2A)$ time complexity for each update. However, the number of updates can be reduced by using the ``batch'' update scheme adapted in \citep{jaksch2010near,dann2015sample} and in this case Algorithm \ref{algorithm} only needs to update the value function $V_{t+1}$ when the number of visits $N_t(s,a)$ doubles. With this update scheme, the number of updates is upper bounded by $O(SA \log T)$ and the total cost for updating the value function is $O(S^3A^2\log T)$. In addition, the Algorithm \ref{algorithm} still needs to choose the action with respect to the value function $V_{t}$ and it costs $O(AT)$ time complexity. Thus, the total computation complexity of the ``batch'' version of Algorithm \ref{algorithm} is $O(AT+S^3A^2\log T)$.}

\begin{algorithm*}[t]
    \caption{Upper Confidence Value-iteration $\algname$}\label{algorithm}
    \begin{algorithmic}[1]
    \STATE Receive state $s_1$ and set initial value function $\qvalue_1(s,a)\leftarrow 1/(1-\gamma)$, $N_0(s,a) =N_0(s,a, s') = N_0(s) \leftarrow 0 $ for all $s\in \cS,a\in \cA, s'\in \cS$
    \FOR{ step $t=1,\ldots$}
    \STATE Let $\pi_t(\cdot) \leftarrow \argmax_{a\in \cA}\qvalue_t(\cdot,a)$, take action $a_t\leftarrow \pi_t(s_t)$ and receive next state $s_{t+1}\sim \PP(\cdot|s_t,a_t)$\label{alg:line1}
    \STATE Set $N_t(s)\leftarrow N_{t-1}(s)$, $N_t(s,a)\leftarrow N_{t-1}(s,a)$ and $N_t(s,a,s')\leftarrow N_{t-1}(s,a,s')$ for all $s\in \cS, a\in \cA, s' \in \cS$
    \STATE Update $N_t(s_t)\leftarrow N_{t}(s_t)+1, N_t(s_t,a_t)\leftarrow N_{t}(s_t,a_t)+1$ and $N_t(s_t,a_t,s_{t+1})\leftarrow N_t(s_t,a_t,s_{t+1})+1$
    \STATE For all $s\in \cS,a\in\cA$, set
    \begin{align}
        \PP_t(s'|s,a)=\frac{N_t(s,a,s')}{N_t(s,a)\vee 1}.\label{eq:empirical_tran}
    \end{align}
    \STATE Update new value function $\qvalue_{t+1}(s,a)$ and $\vvalue_{t+1}(s)$ by
    \begin{align}
        &\qvalue_{t+1}(s,a)=\min\big\{\qvalue_{t}(s,a),\reward(s,a)+\gamma [\PP_t\vvalue_{t}](s,a)+\CC{\gamma}\text{UCB}_t(s,a)\big\},\notag\\
        &\vvalue_{t+1}(s)=\max_{a\in \cA} \qvalue_{t+1}(s,a).
        \label{eq:update}
    \end{align}
    where 
    \begin{align}
        \text{UCB}_t(s,a)&=\sqrt{\frac{8U\text{Var}_{s'\sim \PP_t(\cdot|s,a)}(V_{t}(s'))}{N_t(s,a)\vee 1}}+\frac{8U/(1-\gamma)}{N_t(s,a)\vee 1}\notag\\
        &\qquad +\sqrt{\frac{8\sum_{s'}\PP_t(s'|s,a)\min\big\{100 B_t(s'),1/{(1-\gamma)^2}\big\}}{N_t(s,a)\vee 1}},\label{eq:UCB1}    \end{align}
    and $B_t(s')=\beta/\big[(1-\gamma)^5\big(N_t(s')\vee 1\big)\big].$
    \ENDFOR
    \end{algorithmic}
\end{algorithm*}

\subsection{Regret Analysis}\label{sec:upper}
In this subsection, we provide the regret bound of $\algname$. 
 We first give the formal definition of the regret for the discounted MDP setting. 
\begin{definition}\label{def:regret}
For a given non-stationary policy $\pi$, we define the regret $\text{Regret}(T)$ as follow:
\begin{align}
    \textrm{Regret}(T) = \sum_{t=1}^{T} \big[\vvalue^*(s_t) - \vvalue^{\pi}_t(s_t)\big].\notag
\end{align}
\end{definition}
The same regret has been used in prior work \citep{yang2020q,zhou2020provably,zhou2020nearly} on discounted MDPs. It is related to the ``sample complexity of exploration'' \citep{kakade2003sample,lattimore2012pac,dong2019q}. For more details about the connection between the regret and the sample complexity, please refer to Appendix \ref{section:discussion}.

\begin{remark}
\CCC{Without the use of generative model \citep{kakade2003sample}, an agent may enter bad states at the first few steps in discounted MDPs and there is no ``restarting'' mechanism as in episodic MDPs that can prevent the agent from being stuck in those bad states. Due to this limitation, both the regret and the sample complexity of exploration guarantees are not sufficient to ensure a good policy being learned. We think this is the fundamental limitation in the online learning of discounted MDPs.}
\end{remark}

With Definition \ref{def:regret}, we introduce our main theorem, which gives an upper bound on the regret for $\algname$.
\begin{theorem}\label{thm:1}
Let $U=\log ({40SAT^3\log^2 T}/{(\delta}(1-\gamma)^2))$. If we set $\beta=S^2A^2U^5$ in $\algname$, then with probability at least $1-\delta$,  the regret of $\algname$ in Algorithm \ref{algorithm} is bounded by
\begin{align}
    \text{Regret}(T)
    \leq \frac{752S^2A^{1.5}U^{3.5}}{(1-\gamma)^{3.5}}+\frac{60U\sqrt{SAT}}{(1-\gamma)^{1.5}}+\frac{4\sqrt{TU}}{(1-\gamma)^2}.\notag
\end{align}
\end{theorem}
\begin{remark}
Notice that when $T=\tilde\Omega ({S^3A^2}/{(1-\gamma)^4})$ and $SA=\Omega( {1}/{(1-\gamma)})$, the regret is bounded by $\Tilde{O}\big({\sqrt{SAT}}/{(1-\gamma)^{1.5}}\big) $. In addition, since $\text{Regret}(T)\leq {T}/{(1-\gamma)}$ holds for any $T$, we have $\EE[\text{Regret}(T)]=\Tilde{O}\big(\sqrt{SAT}/{(1-\gamma)^{1.5}}+{T\delta}/{1-\gamma}\big).$
When choosing $\delta=1/{T}$, we have $\EE[\text{Regret}(T)]=\Tilde{O}\big({\sqrt{SAT}}/{(1-\gamma)^{1.5}}\big)$. 
\end{remark}

We also provide a regret lower bound, which suggests that our $\algname$ is nearly minimax optimal.
\begin{theorem}\label{thm:2}
Suppose $\gamma\ge 2/3, A\ge 30$ and $T\ge {100SAL}/{(1-\gamma)^4}$, then for any algorithm, there exists an MDP such that
\begin{align}
    \EE[\text{Regret}(T)]\ge \frac{\sqrt{SAT}}{10000(1-\gamma)^{1.5}}-\frac{4\sqrt{STL}}{(1-\gamma)^{1.5}}-\frac{8S}{(1-\gamma)^2},\notag
\end{align}
where $L=\log{(300S^4T^2/ (1-\gamma))}\log(10ST)$. 
\end{theorem}
\begin{remark}
When $T$ is large enough and $A = \tilde\Omega(1)$, Theorem \ref{thm:2} suggests that the lower bound of regret is $\tilde{\Omega}({\sqrt{SAT}}/{(1-\gamma)^{1.5}})$. It can be seen that the regret of $\algname$ in Theorem \ref{thm:1} matches this lower bound up to logarithmic factors. Therefore, $\algname$ is nearly minimax optimal. 
\end{remark}

\section{Proof of the Main Results}\label{section: sketch}

In this section, we provide the proofs of Theorems \ref{thm:1} and~\ref{thm:2}. The missing proofs are deferred to the appendix. 

\subsection{Proof of Theorem \ref{thm:1}}\label{section: main}
    In this subsection, we prove Theorem \ref{thm:1}. For simplicity, let $\delta'={(1-\gamma)^2\delta}/{(80T\log^2 T)}$, then $U=\log (SAT^2/\delta')$. We first present the following key lemma, which shows that the optimal value functions $\vvalue^*$ and $\qvalue^*$ can be upper bounded by the estimated functions $\vvalue_t$ and $\qvalue_t$ with high probability: 
\begin{lemma}\label{lemma:UCB}
With probability at least $1-64T\delta \log^2 T/(1-\gamma)^2$, for all $t\in[T],s\in \cS,a\in \cA$, we have $\qvalue_t(s,a)\ge \qvalue^*(s,a),\ \vvalue_t(s)\ge \vvalue^*(s)$.
\end{lemma}
Equipped with Lemma \ref{lemma:UCB}, we can decompose the regret of $\algname$ as follows:
\begin{align}
    \text{Regret}(T)
    &\leq \sum_{t=1}^{T} \big[\vvalue_t(s_t) - \vvalue^{\pi}_t(s_t)\big]= 
    \underbrace{\sum_{t=1}^{T} \big[\qvalue_{t}(s_t,a_t) - \qvalue^{\pi}_t(s_t, a_t)\big]}_{{\text{Regret}'(T)}},\notag
\end{align}
where the inequality holds due to Lemma \ref{lemma:UCB}. 
Therefore, it suffices to bound $\text{Regret}'(T)$. We have
\begin{align}
   \text{Regret}'(T)
   &\leq \sum_{t=1}^{T} \Big(\reward(s_t, a_t)  + \gamma[ \PP_{t-1}\vvalue_{t-1}](s_{t},a_{t})+\CC{\gamma}\text{UCB}_{t-1}(s_{t},a_{t})\notag\\
   &\qquad -\reward(s_t, a_t) - \gamma [\PP\vvalue^{\pi}_{t+1}](s_t,a_t)\Big)\notag\\
    & = \sum_{t=1}^{T} \Big(\gamma[\PP_{t-1}\vvalue_{t-1}](s_{t},a_{t})+\CC{\gamma}\text{UCB}_{t-1}(s_{t},a_{t})-\gamma [\PP\vvalue^{\pi}_{t+1}](s_t,a_t)\Big),\notag
\end{align}
where the inequality holds due to the update rule \eqref{eq:update} and the Bellman equation $\qvalue^{\pi}_t(s_t, a_t) = \reward(s_t, a_t) + \gamma [\PP\vvalue^{\pi}_{t+1}](s_t, a_t)$.
We further have 
\begin{align}
    & \sum_{t=1}^{T} \Big(\gamma[\PP_{t-1}\vvalue_{t-1}](s_{t},a_{t})+\CC{\gamma}\text{UCB}_{t-1}(s_{t},a_{t})-\gamma [\PP\vvalue^{\pi}_{t+1}](s_t,a_t)\Big)\notag\\
    &=  \underbrace{\sum_{t=1}^{T}\gamma(\vvalue_{t-1}(s_{t+1})-\vvalue^{\pi}_{t+1}(s_{t+1}))}_{I_1}+\underbrace{\sum_{t=1}^{T}\gamma\big[(\PP_{t-1}-\PP)(\vvalue_{t-1} -\vvalue^*)\big] (s_t,a_t)}_{I_2}\notag\\
    &\quad + \underbrace{\sum_{t=1}^{T}\gamma[(\PP_{t-1}-\PP)\vvalue^*](s_t,a_t)}_{I_3}    +\underbrace{\sum_{t=1}^{T} \CC{\gamma}\text{UCB}_{t-1}(s_t,a_t)}_{I_4}\notag\\
    &\quad +\underbrace{\sum_{t=1}^{T}\gamma\big[\PP(\vvalue_{t-1}-\vvalue^{\pi}_{t+1})\big](s_t,a_t)-\gamma\big[\vvalue_{t-1}(s_{t+1})-\vvalue^{\pi}_{t+1}(s_{t+1})\big]}_{I_5}.
    \label{eq:sum-5-term}
\end{align}
\CCC{In the above decomposition, term $I_1$ controls the estimation error between the value functions $V_{t-1}$ and $V_{t+1}^{\pi}$, terms $I_2$ and $I_3$ measure the estimation error between the  transition probability function  $\PP$ and the estimated transition probability function $\PP_{t-1}$, term $I_4$ comes from the exploration bonus in Algorithm \ref{algorithm}, and term $I_5$ accounts for the randomness in the stochastic transition process, which can be controlled by the third term $O(\sqrt{TU}/(1-\gamma)^2)$ in Theorem \ref{thm:1}. }

In the remaining of the proof, it suffices to bound terms $I_1$ to $I_5$ separately. 

First, $I_1$ can be regarded as the difference between the estimated $\vvalue_{t-1}$ and the value function $\vvalue^{\pi}_{t+1}$ of policy $\pi$, and it can be bounded by the following lemma. 
\begin{lemma}\label{theorem: I_1}
 For the term $I_1$, We have $I_1\leq \gamma\text{Regret}'(T)+{(2S+2)\gamma}/{1-\gamma}$
\end{lemma}
Next, $I_2$ can be regarded as the ``correction" term between the estimated $\vvalue_{t-1}$ and the optimal value function $\vvalue^*$. It can be bounded by the following lemma.
\begin{lemma}\label{theorem: J_1}
With probability at least $1-64T\delta \log^2 T/(1-\gamma)^2-3\delta$, we have
\begin{align}
    I_2&\leq  (1-\gamma)\text{Regret}'(T)/2+\sqrt{2T\CC{\log(1/\delta)}} +\frac{5S^2A\CC{\log(ST/\delta)\log(3T)}}{(1-\gamma)^2}.\notag
\end{align}
\end{lemma}
In addition, $I_3$ can be regarded as the error between the empirical probability distribution $\PP_{t-1}$ and the true transition probability $\PP$. Note that $\vvalue^*$ is a fixed value function that does not have any randomness. Therefore, $I_3$ can be bounded through the standard concentration inequalities, and its upper bound is presented in the following lemma. \begin{lemma}\label{theorem: J_2}
With probability at least $1-2\delta-\delta/(1-\gamma)$, we have
\begin{align}
I_3&\leq \frac{2SAU^2}{1-\gamma}\notag+U\sqrt{2 SA}\sqrt{\frac{5T}{1-\gamma}+\frac{29U}{3(1-\gamma)^3}+\frac{2\text{Regret}’(T)}{1-\gamma}+\frac{\sqrt{2TU}}{(1-\gamma)^2}}.\notag
\end{align}
\end{lemma}
Furthermore, $I_4$ can be regarded as the summation of the UCB terms, which is also the dominating term of the total regret. It can be bounded by the following lemma. 
 \begin{lemma}\label{theorem: I_5}
With probability at least $1-4\delta-\delta/(1-\gamma)$, we have
\begin{align}
I_4&\leq\frac{37S^2A^{1.5}U^{3.5}}{(1-\gamma)^{2.5}}+ U\sqrt{8 SA}\sqrt{\frac{5T}{1-\gamma}+\frac{29U}{3(1-\gamma)^3}+\frac{2\text{Regret}’(T)}{1-\gamma}+\frac{12SU\sqrt{AT}}{(1-\gamma)^2}}.\notag
\end{align}
\end{lemma}
Finally, $I_5$ is the summation of a martingale difference sequence. By Azuma-Hoeffding inequality, with probability at least $1-\delta$, we have
\begin{align}
    I_5&\leq \frac{\sqrt{2T\log(1/\delta)}}{1-\gamma}.\label{eq:I_3}
\end{align} 
Substituting the upper bounds of terms $I_1$ to $I_5$ from Lemma~\ref{theorem: I_1} to Lemma \ref{theorem: I_5}, as well as \eqref{eq:I_3}, into \eqref{eq:sum-5-term}, and taking a union bound to let all the events introduced in Lemma \ref{theorem: I_1} to Lemma \ref{theorem: I_5} and \eqref{eq:I_3} hold, we have with probability at least $1-{20TU^2\delta}/{(1-\gamma)^2}$, the following inequality holds: 
\begin{align}
    (1-\gamma)\text{Regret}'(T)
     \leq \frac{160S^2A^{1.5}U^{3.5}}{(1-\gamma)^{2.5}}+\frac{54U\sqrt{SAT}}{\sqrt{1-\gamma}}+\frac{2\sqrt{2TU}}{1-\gamma}+12U\sqrt{\frac{SA \text{Regret}'(T)}{1-\gamma}}.\label{eq:xx1}
\end{align}
Using the fact that $x \leq a+b\sqrt{x}\Rightarrow x \leq 1.1a+ 4b^2$, \eqref{eq:xx1} can be further bounded as follows
\begin{align}
    \text{Regret}(T)&\leq \text{Regret}'(T)\notag\\
    &\leq \frac{752S^2A^{1.5}U^{3.5}}{(1-\gamma)^{3.5}}+\frac{60U\sqrt{SAT}}{(1-\gamma)^{1.5}}+\frac{4\sqrt{TU}}{(1-\gamma)^2}.\notag
\end{align}
This completes our proof.
\subsection{{Proof of Theorem \ref{thm:2}}}\label{section: second}

\begin{figure*}[!h]
    \centering
    \begin{tikzpicture}[node distance=1.5cm,>=stealth',bend angle=45,auto]

  \tikzstyle{place}=[circle,thick,draw=gray!75,fill=gray!50,minimum size=6mm]
\tikzset{every loop/.style={min distance=15mm, font=\footnotesize}}
  \begin{scope}[xshift=-3.5cm]
    
    \node [place, label] (c0){$\state_{1,0}$};
    \coordinate [right of = c0, label = center:{$\vdots\ \ \vdots\ \ \vdots$}] (c1) {};
    \node [place] (c21) [right of=c1]{$\state_{1,1}$};
    \path[->] (c0)
    edge [in=120,out=60, draw=blue!100] node[below]{$1-\gamma+\epsilon$} (c21)
    edge [in=170,out=110, min distance = 6cm, loop, draw=blue!100] node[above] {$\gamma - \epsilon$} ()
    edge [in=-120,out=-60] node[above]{$1-\gamma$} (c21)
    edge [in=-170,out=-110, min distance = 6cm, loop] node[above] {$\gamma$} ()
    ;
    \coordinate [right of = c21] (c2) {};
    \node [place] (c3) [right of=c2]{$\state_{2,0}$};
    \path[->] (c21)
    edge [in=180,out=0, , draw=blue!100] node[above]{$1-\gamma$} (c3)
    edge [in=120,out=60, min distance = 6cm, loop, , draw=blue!100] node[above] {$ \gamma$} ()
    ;
    \coordinate[right of=c3, label = center:{$\vdots\ \ \vdots\ \ \vdots$}] (e1) {}; 
    \node [place] [right of=e1] (c4){$\state_{2,1}$};
    \path[->] (c3)
    edge [in=120,out=60, , draw=blue!100] node[below]{} (c4)
    edge [in=170,out=110, min distance = 6cm, loop, , draw=blue!100] node[above] {} ()
    edge [in=-120,out=-60] node[above]{} (c4)
    edge [in=-170,out=-110, min distance = 6cm, loop] node[above] {} ()
    ;
    \coordinate[below of=c4] (c5) {};
    \path[->] (c4)
    edge [in=90,out=270, , draw=blue!100] node[above]{} (c5)
    ;
    \node [place] [below of=c5] (c6){$\state_{S-1,0}$};
    \path[->] (c5)
    edge [in=90,out=270, , draw=blue!100] node[above]{} (c6)
    ;
    \coordinate[left of=c6, label = center:{$\vdots\ \ \vdots\ \ \vdots$}] (e10) {};
    \node [place] [left of=e10] (c7){$\state_{S-1, 1}$};
     \path[->] (c6)
    edge [in=60,out=120, , draw=blue!100] node[below]{} (c7)
    edge [in=10,out=70, min distance = 6cm, loop, , draw=blue!100] node[above] {} ()
    edge [in=-60,out=-120] node[above]{} (c7)
    edge [in=-10,out=-70, min distance = 6cm, loop] node[above] {} ()
    ;
    \coordinate[left of=c7] (e11) {};
    \node [place] [left of=e11] (c8){$\state_{S,0}$};
    \path[->] (c7)
    edge [in=0,out=180, , draw=blue!100] node[above]{} (c8)
    edge [in=60,out=120, min distance = 6cm, loop, , draw=blue!100] node[above] {} ()
    ;
        \coordinate[left of=c8, label = center:{$\vdots\ \ \vdots\ \ \vdots$}] (e12) {};
    \node [place] [left of=e12] (c9){$\state_{S,1}$};
     \path[->] (c8)
    edge [in=60,out=120, , draw=blue!100] node[below]{} (c9)
    edge [in=10,out=70, min distance = 6cm, loop, , draw=blue!100] node[above] {} ()
    edge [in=-60,out=-120] node[above]{} (c9)
    edge [in=-10,out=-70, min distance = 6cm, loop] node[above] {} ()
    ;
    \path[->](c9)
    edge [in=270,out=90, , draw=blue!100] node[above]{} (c0)
    edge [in=150,out=-150, min distance = 6cm, loop, , draw=blue!100] node[above] {} ()
    ;
  \end{scope}

\end{tikzpicture}
\vspace*{-14mm}
    \caption{A class of hard-to-learn MDPs considered in Theorem \ref{thm:2}. The MDP can be regarded as a combination of $S$ two-state MDPs, each of which is an MDP illustrated on the top-left corner. In addition, the $i$-th two-state MDP has the $a_i^*$-th action as its optimal action. The blue arrows represent the optimal actions in different states. $\epsilon = \sqrt{A(1-\gamma)/K}/24$. }
    \label{fig:hardmdp}
\end{figure*}
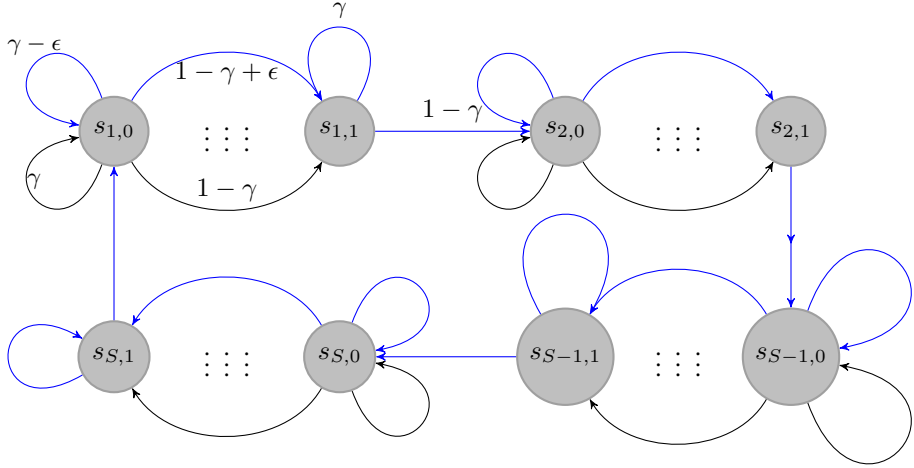

In this subsection, we provide the proof of Theorem \ref{thm:2}. The proof of the lower bound is based on constructing a class of hard MDPs. Specifically, the state space $\cS$ consists of $2S$ states $\{s_{i,0},s_{i,1}\}_{i\in[S]}$ and the action space $\cA$ contains $A$ actions. The reward function $r$ satisfies that $r(s_{i,0},a)=0$ and $r(s_{i,1},a)=1$ for any $a\in\cA,i\in[S]$. The probability transition function $\PP$ is defined as follows.
\begin{align}
    &\PP(s_{i,1}|s_{i,0},a)=1-\gamma+\ind_{a=a^*_i} \frac{1}{24}\sqrt{\frac{A(1-\gamma)}{K}},\PP(s_{i,1}|s_{i,1},a)=\gamma, \notag\\
    &\PP(s_{i,0}|s_{i,0},a)=\gamma-\ind_{a=a^*_i} \frac{1}{24}\sqrt{\frac{A(1-\gamma)}{K}},\PP(s_{i+1,0}|s_{i,1},a)=1-\gamma, \notag
\end{align}
where we assume $s_{S+1,0}=s_{1,0}$ for simplicity and $a^*_i$ is the optimal action for state $s_{i,0}$. The MDP is illustrated in Figure \ref{fig:hardmdp}, which can be regarded as $S$ copies of  the ``single" two-state MDP arranged in a circle. The two-state MDP is the same as that proposed in \cite{liu2020regret}. Each of the two-state MDP has two states and one ``optimal" action $a_i^*$ satisfied $\PP(s_{i,1}|s_{i,0},a_i^*)=1-\gamma+\epsilon$. 
Compared with the MDP instance in \cite{jaksch2010near}, both instances use $S$ copies of a single MDP. However, unlike the MDP in \cite{jaksch2010near}  which only has one ``optimal" action among all $SA$ actions, our MDP which has in total $S$ ``optimal" actions, which makes it harder to analyze. 

Now we begin to prove our lower bound. Let $\EE_{\ab^*}[\cdot]$ denote the expectation conditioned on one fixed selection of $\ab^* = (a_1^*,\dots, a_S^*)$. We introduce a shorthand notation $\EE^*$ to denote $\EE^*[\cdot] = 1/A^S\cdot\sum_{\ab^* \in \cA^S}\EE_{\ab^*}[\cdot]$. 
Here $\EE^*$ is the average value of expectation over the randomness from MDP defined by different optimal actions. From now on, we aim to lower bound $\EE^*[\text{Regret}(T)]$, since once $\EE^*[\text{Regret}(T)]$ is lower bounded, $\EE[\text{Regret}(T)]$ can be lower bounded by selecting $a_1^*,\dots, a_S^*$ which maximizes  $\EE[\text{Regret}(T)]$. We set $T = 10SK$ in the following proof. Based on the definition of $\EE^*$, we have the following lemma.
\begin{lemma}\label{lemma:transition}
The expectated regret $\EE^*[\text{Regret}(T)]$ can be lower bounded as follows:
\begin{align}
\EE^*[\text{Regret}(T)]\ge \EE^*\bigg[\sum_{t=1}^{T} \vvalue^*(s_t) - \frac{\reward(s_t,a_t)}{1-\gamma}\bigg]-\frac{4}{(1-\gamma)^2}.\notag
\end{align}
\end{lemma}
By Lemma \ref{lemma:transition}, it suffices to lower bound  $\sum_{t=1}^{T} [\vvalue^*(s_t) - \reward(s_t,a_t)/(1-\gamma)]$, which is $\text{Regret}^{\text{Liu}}(T)$ defined in \cite{liu2020regret}. When an agent visits the state set $\{s_{j,0}, s_{j,1}\}$ for the $i$-th time, we denote the state in $\{s_{j,0}, s_{j,1}\}$ it visited as $X_{j,i}$, and the following action selected by the agent as $A_{j,i}$. 
Let $T_j$ be the number of steps for the agent staying in $\{s_{j,0}, s_{j,1}\}$ in the total $T$ steps. Then the regret can be further decomposed as follows: 
\begin{align}
\EE^*\bigg[\sum_{t=1}^{T} \vvalue^*(s_t) - \frac{\reward(s_t,a_t)}{1-\gamma}\bigg]
&=\sum_{j=1}^{S} \EE^*\bigg[\sum_{i=1}^{T_j}  \vvalue^*(X_{j,i})-\frac{\reward(X_{j,i},A_{j,i})}{1-\gamma}\bigg]=I_1 + I_2 + I_3,\notag
\end{align}
where 
\begin{align}
    &I_1 = \sum_{j=1}^{S} \EE^*\bigg[\sum_{i=1}^{K}  \vvalue^*(X_{j,i})-\frac{\reward(X_{j,i},A_{j,i})}{1-\gamma}\bigg],\notag \\
    &I_2 = \sum_{j=1}^{S} \EE^*\bigg[\sum_{i=K+1}^{T_j}  \vvalue^*(X_{j,i})-\frac{\reward(X_{j,i},A_{j,i})}{1-\gamma}\bigg| T_j> K\bigg] \cdot \PP^*[T_j> K],\notag \\
    &I_3 = -\sum_{j=1}^{S} \EE^*\bigg[\sum_{i=T_j+1}^{K}  \vvalue^*(X_{j,i})-\frac{\reward(X_{j,i},A_{j,i})}{1-\gamma}\bigg| T_j< K\bigg]\cdot \PP^*[T_j< K].\notag
\end{align}
Note that $I_1$ essentially represents the regret over $S$ two-state MDPs in their first $K$ steps, and it can be lower bounded through the following lemma. 
\begin{lemma}\label{lemma:suhao}
If $K\ge 10SA/(1-\gamma)^4$, then for each $j\in [S]$, we have
\begin{align}
    &\EE^*\bigg[\sum_{i=1}^K (1-\gamma)\vvalue^*(X_{j,i})-\reward(X_{j,i},A_{j,i})\bigg]\ge \frac{\sqrt{AK}}{2304\sqrt{1-\gamma}}-\frac{1}{1-\gamma} .\notag
\end{align}
\end{lemma}
This lemma shows that the expected regret of first $K$ steps on states $s_{j,0}$ and $s_{j,1}$ is at least $ \tilde{\Omega}\big(\sqrt{AK}/(1-\gamma)^{0.5}-1/(1-\gamma)\big)$. Therefore by Lemma \ref{lemma:suhao}, we have
\begin{align}
    I_1 &= \sum_{j=1}^{S} \EE^*\bigg[\sum_{i=1}^{K}  \vvalue^*(X_{j,i})-\frac{\reward(X_{j,i},A_{j,i})}{1-\gamma}\bigg]\ge  \frac{\sqrt{SAT}}{2304\sqrt{10}(1-\gamma)^{1.5}}-\frac{S}{(1-\gamma)^2}.\label{eq:i1}
\end{align}
To bound $I_2$, we need the following lemma. 
\begin{lemma}\label{lemma:remain}
With probability  at least $1-2ST\delta\log T/(1-\gamma)$, for each $j\in[S]$ and $K+1 \leq t \leq T$, we have
\begin{align}
    &\sum_{i=K+1}^{t} \vvalue^*(X_{j,i})-\frac{\reward(X_{j,i},A_{j,i})}{1-\gamma}\ge - \frac{\sqrt{2t\log ({1}/{\delta})\log T}}{(1-\gamma)^{1.5}}-\frac{4}{(1-\gamma)^{2}}.\notag
\end{align}
\end{lemma}
Lemma \ref{lemma:remain} gives a crude lower bound of $I_2$. Taking expectation over Lemma \ref{lemma:remain} and taking summation over all states, we have
\begin{align}
    I_2&\ge \sum_{j=1}^{S}  \EE^*\bigg[\bigg(-\frac{\sqrt{2T_j\log ({1}/{\delta})\log T}}{(1-\gamma)^{1.5}} - \frac{4}{(1-\gamma)^2}\bigg)      \bigg|T_j> K\bigg]\PP^*[T_j> K]\notag\\
    &\qquad - \sum_{j=1}^{S} \frac{T}{1-\gamma}\cdot \frac{2ST\delta\log T}{(1-\gamma)^2}\notag \\
    &\ge \sum_{j=1}^{S}  \EE^*\bigg[-\frac{\sqrt{2T_j\log ({1}/{\delta})\log T}}{(1-\gamma)^{1.5}}\bigg] - \frac{4S}{(1-\gamma)^2}- \frac{2S^2T^2\delta\log T}{(1-\gamma)^2} \notag\\
    &\ge \sum_{j=1}^{S}  -\frac{\sqrt{2\EE^*[T_j]\log ({1}/{\delta})\log T}}{(1-\gamma)^{1.5}} - \frac{4S}{(1-\gamma)^2}- \frac{2S^2T^2\delta\log T}{(1-\gamma)^2}\notag\\
    &\ge -\frac{\sqrt{2ST\log ({1}/{\delta})\log T}}{(1-\gamma)^{1.5}}-\frac{4S}{(1-\gamma)^2}-\frac{2S^2T^2\delta\log T}{(1-\gamma)^2},\label{eq:i2}
 \end{align}
where the first inequality holds due to Lemma \ref{lemma:remain}, the second inequality holds since $1-2ST\delta\log T/(1-\gamma) \leq 1$ and $\EE[-X|Y]\PP(Y) \geq \EE[-X]$ when $X \geq 0$, the third inequality holds due to Jensen's inequality and the fact that $\sqrt{x}$ is a concave function, and the last inequality holds due to Jensen's inequality and the fact that $\sum_{j=1}^S \EE^*[T_j]=T$. To bound $I_3$, we need the following lemma, which suggests that when $K$ is large enough, $T_i>K$ happens with high probability: 
\begin{lemma}\label{lemma:length}
When $K\ge 10 A\log (1/\delta)/{(1-\gamma)^4},$ with probability at least $1-2S\delta$, for all $i\in [S]$, we have $T_i >K$.
\end{lemma}
Notice that the difference of transition probability between the optimal action and suboptimal actions is $\sqrt{A(1-\gamma)}/24K$. In this case, when $T$ is large enough, $T_i$ is close to $T/S=10K$. Thus $I_3$ can be lower bounded as follows:
\begin{align}
    I_3&\ge  -\sum_{j=1}^{S}\frac{K}{1-\gamma}\PP^*[T_j< K]\ge -\frac{ST\delta}{5(1-\gamma)},\label{eq:i3}
\end{align}
where the first inequality holds due to $0\leq \reward(X_{j,i},A_{j,i})\leq 1$ and the second inequality holds due to Lemma \ref{lemma:length}. Finally, setting $\delta={1}/\big({4ST^2(1-\gamma)^2 \log T}\big)$, we can verify that the requirements of $K$ in Lemma \ref{lemma:suhao} and Lemma \ref{lemma:length} hold when $T$ satisfies $T\ge {100SAL}/{(1-\gamma)^4}$, and $L= \log {(300S^4T^2/ ((1-\gamma)^2\delta))}\log T$. Therefore, substituting $\delta={1}/\big({4ST^2(1-\gamma)^2 \log T}\big)$ into \eqref{eq:i2} and \eqref{eq:i3}, and combining \eqref{eq:i1}, \eqref{eq:i2}, \eqref{eq:i3} and Lemma \ref{lemma:transition}, we have
\begin{align}
    \EE[\text{Regret}(T)]\ge \frac{\sqrt{SAT}}{10000(1-\gamma)^{1.5}}-\frac{4\sqrt{STL}}{(1-\gamma)^{1.5}}-\frac{8S}{(1-\gamma)^2},\notag
\end{align}
which completes the proof of Theorem \ref{thm:2}.

\begin{section}{Conclusions and Future Work}\label{section: conclusion}
We proposed $\algname$, an online RL algorithm for discounted tabular MDPs. We show that the regret of $\algname$ can be upper bounded by $\tilde O(\sqrt{SAT}/(1-\gamma)^{1.5})$ and we prove a matching lower bound on the expected regret $\tilde \Omega(\sqrt{SAT}/(1-\gamma)^{1.5})$. There is still a gap between the upper and lower bounds when $T\leq \max\{S^3A^2/(1-\gamma)^4, SA/(1-\gamma)^4\}$, and we leave it as an open problem for future work.
\end{section}

\section*{Acknowledgments and Disclosure of Funding}
We thank Csaba Szepesv\'ari for a valuable suggestion on improving the presentation of the proof.
We thank the anonymous reviewers for their helpful comments. 
JH, DZ and QG are partially supported by the National Science Foundation CAREER Award 1906169, IIS-1904183, BIGDATA IIS-1855099 and AWS Machine Learning Research Award. The views and conclusions contained in this paper are those of the authors and should not be interpreted as representing any funding agencies.

\bibliographystyle{ims}
\bibliography{reference}

\begin{thebibliography}{35}
\expandafter\ifx\csname natexlab\endcsname\relax\def\natexlab#1{#1}\fi
\expandafter\ifx\csname url\endcsname\relax
  \def\url#1{\texttt{#1}}\fi
\expandafter\ifx\csname urlprefix\endcsname\relax\def\urlprefix{URL }\fi

\bibitem[{Agarwal et~al.(2019)Agarwal, Kakade and Yang}]{agarwal2019model}
\textsc{Agarwal, A.}, \textsc{Kakade, S.} and \textsc{Yang, L.~F.} (2019).
\newblock Model-based reinforcement learning with a generative model is minimax
  optimal.
\newblock \textit{arXiv preprint arXiv:1906.03804} .

\bibitem[{Ayoub et~al.(2020)Ayoub, Jia, Szepesvari, Wang and
  Yang}]{ayoub2020model}
\textsc{Ayoub, A.}, \textsc{Jia, Z.}, \textsc{Szepesvari, C.}, \textsc{Wang,
  M.} and \textsc{Yang, L.~F.} (2020).
\newblock Model-based reinforcement learning with value-targeted regression.
\newblock \textit{arXiv preprint arXiv:2006.01107} .

\bibitem[{Azar et~al.(2013)Azar, Munos and Kappen}]{azar2013minimax}
\textsc{Azar, M.~G.}, \textsc{Munos, R.} and \textsc{Kappen, H.~J.} (2013).
\newblock Minimax pac bounds on the sample complexity of reinforcement learning
  with a generative model.
\newblock \textit{Machine learning} \textbf{91} 325--349.

\bibitem[{Azar et~al.(2017)Azar, Osband and Munos}]{azar2017minimax}
\textsc{Azar, M.~G.}, \textsc{Osband, I.} and \textsc{Munos, R.} (2017).
\newblock Minimax regret bounds for reinforcement learning.
\newblock In \textit{Proceedings of the 34th International Conference on
  Machine Learning-Volume 70}. JMLR. org.

\bibitem[{Brafman and Tennenholtz(2002)}]{brafman2002r}
\textsc{Brafman, R.~I.} and \textsc{Tennenholtz, M.} (2002).
\newblock R-max-a general polynomial time algorithm for near-optimal
  reinforcement learning.
\newblock \textit{Journal of Machine Learning Research} \textbf{3} 213--231.

\bibitem[{Cesa-Bianchi and Lugosi(2006)}]{cesa2006prediction}
\textsc{Cesa-Bianchi, N.} and \textsc{Lugosi, G.} (2006).
\newblock \textit{Prediction, learning, and games}.
\newblock Cambridge university press.

\bibitem[{Dann and Brunskill(2015)}]{dann2015sample}
\textsc{Dann, C.} and \textsc{Brunskill, E.} (2015).
\newblock Sample complexity of episodic fixed-horizon reinforcement learning.
\newblock In \textit{Advances in Neural Information Processing Systems}.

\bibitem[{Dann et~al.(2019)Dann, Li, Wei and Brunskill}]{dann2019policy}
\textsc{Dann, C.}, \textsc{Li, L.}, \textsc{Wei, W.} and \textsc{Brunskill, E.}
  (2019).
\newblock Policy certificates: Towards accountable reinforcement learning.
\newblock In \textit{International Conference on Machine Learning}. PMLR.

\bibitem[{Dong et~al.(2019)Dong, Wang, Chen and Wang}]{dong2019q}
\textsc{Dong, K.}, \textsc{Wang, Y.}, \textsc{Chen, X.} and \textsc{Wang, L.}
  (2019).
\newblock Q-learning with ucb exploration is sample efficient for
  infinite-horizon mdp.
\newblock \textit{arXiv preprint arXiv:1901.09311} .

\bibitem[{Jaksch et~al.(2010)Jaksch, Ortner and Auer}]{jaksch2010near}
\textsc{Jaksch, T.}, \textsc{Ortner, R.} and \textsc{Auer, P.} (2010).
\newblock Near-optimal regret bounds for reinforcement learning.
\newblock \textit{Journal of Machine Learning Research} \textbf{11} 1563--1600.

\bibitem[{Jin et~al.(2018)Jin, Allen-Zhu, Bubeck and Jordan}]{jin2018q}
\textsc{Jin, C.}, \textsc{Allen-Zhu, Z.}, \textsc{Bubeck, S.} and
  \textsc{Jordan, M.~I.} (2018).
\newblock Is q-learning provably efficient?
\newblock In \textit{Advances in Neural Information Processing Systems}.

\bibitem[{Kakade et~al.(2003)}]{kakade2003sample}
\textsc{Kakade, S.~M.} \textsc{et~al.} (2003).
\newblock \textit{On the sample complexity of reinforcement learning}.
\newblock Ph.D. thesis, University of London London, England.

\bibitem[{Kearns and Singh(1999)}]{kearns1999finite}
\textsc{Kearns, M.~J.} and \textsc{Singh, S.~P.} (1999).
\newblock Finite-sample convergence rates for q-learning and indirect
  algorithms.
\newblock In \textit{Advances in neural information processing systems}.

\bibitem[{Lattimore and Hutter(2012)}]{lattimore2012pac}
\textsc{Lattimore, T.} and \textsc{Hutter, M.} (2012).
\newblock Pac bounds for discounted mdps.
\newblock In \textit{International Conference on Algorithmic Learning Theory}.
  Springer.

\bibitem[{Liu and Su(2020)}]{liu2020regret}
\textsc{Liu, S.} and \textsc{Su, H.} (2020).
\newblock Regret bounds for discounted mdps.

\bibitem[{Maurer and Pontil(2009)}]{maurer2009empirical}
\textsc{Maurer, A.} and \textsc{Pontil, M.} (2009).
\newblock Empirical bernstein bounds and sample variance penalization.
\newblock \textit{arXiv preprint arXiv:0907.3740} .

\bibitem[{Neu and Pike-Burke(2020)}]{neu2020unifying}
\textsc{Neu, G.} and \textsc{Pike-Burke, C.} (2020).
\newblock A unifying view of optimism in episodic reinforcement learning.
\newblock \textit{arXiv preprint arXiv:2007.01891} .

\bibitem[{Osband and Van~Roy(2016)}]{osband2016lower}
\textsc{Osband, I.} and \textsc{Van~Roy, B.} (2016).
\newblock On lower bounds for regret in reinforcement learning.
\newblock \textit{arXiv preprint arXiv:1608.02732} .

\bibitem[{Osband and Van~Roy(2017)}]{osband2017posterior}
\textsc{Osband, I.} and \textsc{Van~Roy, B.} (2017).
\newblock Why is posterior sampling better than optimism for reinforcement
  learning?
\newblock In \textit{International Conference on Machine Learning}.

\bibitem[{Pacchiano et~al.(2020)Pacchiano, Ball, Parker-Holder, Choromanski and
  Roberts}]{pacchiano2020optimism}
\textsc{Pacchiano, A.}, \textsc{Ball, P.}, \textsc{Parker-Holder, J.},
  \textsc{Choromanski, K.} and \textsc{Roberts, S.} (2020).
\newblock On optimism in model-based reinforcement learning.
\newblock \textit{arXiv preprint arXiv:2006.11911} .

\bibitem[{Russo(2019)}]{russo2019worst}
\textsc{Russo, D.} (2019).
\newblock Worst-case regret bounds for exploration via randomized value
  functions.
\newblock In \textit{Advances in Neural Information Processing Systems}.

\bibitem[{Sidford et~al.(2018{\natexlab{a}})Sidford, Wang, Wu, Yang and
  Ye}]{sidford2018near}
\textsc{Sidford, A.}, \textsc{Wang, M.}, \textsc{Wu, X.}, \textsc{Yang, L.~F.}
  and \textsc{Ye, Y.} (2018{\natexlab{a}}).
\newblock Near-optimal time and sample complexities for for solving discounted
  markov decision process with a generative model.
\newblock \textit{arXiv preprint arXiv:1806.01492} .

\bibitem[{Sidford et~al.(2018{\natexlab{b}})Sidford, Wang, Wu and
  Ye}]{sidford2018variance}
\textsc{Sidford, A.}, \textsc{Wang, M.}, \textsc{Wu, X.} and \textsc{Ye, Y.}
  (2018{\natexlab{b}}).
\newblock Variance reduced value iteration and faster algorithms for solving
  markov decision processes.
\newblock In \textit{Proceedings of the Twenty-Ninth Annual ACM-SIAM Symposium
  on Discrete Algorithms}. SIAM.

\bibitem[{Simchowitz and Jamieson(2019)}]{simchowitz2019non}
\textsc{Simchowitz, M.} and \textsc{Jamieson, K.~G.} (2019).
\newblock Non-asymptotic gap-dependent regret bounds for tabular mdps.
\newblock In \textit{Advances in Neural Information Processing Systems}.

\bibitem[{Strehl et~al.(2006)Strehl, Li, Wiewiora, Langford and
  Littman}]{strehl2006pac}
\textsc{Strehl, A.~L.}, \textsc{Li, L.}, \textsc{Wiewiora, E.},
  \textsc{Langford, J.} and \textsc{Littman, M.~L.} (2006).
\newblock Pac model-free reinforcement learning.
\newblock In \textit{Proceedings of the 23rd international conference on
  Machine learning}.

\bibitem[{Strehl and Littman(2008)}]{strehl2008analysis}
\textsc{Strehl, A.~L.} and \textsc{Littman, M.~L.} (2008).
\newblock An analysis of model-based interval estimation for markov decision
  processes.
\newblock \textit{Journal of Computer and System Sciences} \textbf{74}
  1309--1331.

\bibitem[{Szita and Szepesv{\'a}ri(2010)}]{szita2010model}
\textsc{Szita, I.} and \textsc{Szepesv{\'a}ri, C.} (2010).
\newblock Model-based reinforcement learning with nearly tight exploration
  complexity bounds .

\bibitem[{Wainwright(2019)}]{wainwright2019variance}
\textsc{Wainwright, M.~J.} (2019).
\newblock Variance-reduced $ q $-learning is minimax optimal.
\newblock \textit{arXiv preprint arXiv:1906.04697} .

\bibitem[{Wang(2017)}]{wang2017randomized}
\textsc{Wang, M.} (2017).
\newblock Randomized linear programming solves the discounted markov decision
  problem in nearly-linear running time.
\newblock \textit{arXiv preprint arXiv:1704.01869} .

\bibitem[{Yang et~al.(2021)Yang, Yang and Du}]{yang2020q}
\textsc{Yang, K.}, \textsc{Yang, L.} and \textsc{Du, S.} (2021).
\newblock Q-learning with logarithmic regret  1576--1584.

\bibitem[{Zanette and Brunskill(2019)}]{zanette2019tighter}
\textsc{Zanette, A.} and \textsc{Brunskill, E.} (2019).
\newblock Tighter problem-dependent regret bounds in reinforcement learning
  without domain knowledge using value function bounds.
\newblock \textit{arXiv preprint arXiv:1901.00210} .

\bibitem[{Zhang et~al.(2020{\natexlab{a}})Zhang, Zhou and Ji}]{zhang2020almost}
\textsc{Zhang, Z.}, \textsc{Zhou, Y.} and \textsc{Ji, X.} (2020{\natexlab{a}}).
\newblock Almost optimal model-free reinforcement learning via
  reference-advantage decomposition.
\newblock \textit{arXiv preprint arXiv:2004.10019} .

\bibitem[{Zhang et~al.(2020{\natexlab{b}})Zhang, Zhou and Ji}]{zhang2020model}
\textsc{Zhang, Z.}, \textsc{Zhou, Y.} and \textsc{Ji, X.} (2020{\natexlab{b}}).
\newblock Model-free reinforcement learning: from clipped pseudo-regret to
  sample complexity.
\newblock \textit{arXiv preprint arXiv:2006.03864} .

\bibitem[{Zhou et~al.(2021{\natexlab{a}})Zhou, Gu and
  Szepesvari}]{zhou2020nearly}
\textsc{Zhou, D.}, \textsc{Gu, Q.} and \textsc{Szepesvari, C.}
  (2021{\natexlab{a}}).
\newblock Nearly minimax optimal reinforcement learning for linear mixture
  markov decision processes.
\newblock In \textit{COLT}.

\bibitem[{Zhou et~al.(2021{\natexlab{b}})Zhou, He and Gu}]{zhou2020provably}
\textsc{Zhou, D.}, \textsc{He, J.} and \textsc{Gu, Q.} (2021{\natexlab{b}}).
\newblock Provably efficient reinforcement learning for discounted mdps with
  feature mapping.
\newblock In \textit{International Conference on Machine Learning}. PMLR.

\end{thebibliography}
\section*{Checklist}


\begin{enumerate}

\item For all authors...
\begin{enumerate}
  \item Do the main claims made in the abstract and introduction accurately reflect the paper's contributions and scope?
    \answerYes{}
  \item Did you describe the limitations of your work?
    \answerYes{We discuss the limitations and potential future works in Section \ref{section: conclusion}}
  \item Did you discuss any potential negative societal impacts of your work?
    \answerNA{Our work provides a theoretical analysis of discounted MDPs, and there is no potential negative social impact.}
  \item Have you read the ethics review guidelines and ensured that your paper conforms to them?
    \answerYes{}
\end{enumerate}

\item If you are including theoretical results...
\begin{enumerate}
  \item Did you state the full set of assumptions of all theoretical results?
    \answerYes{}
	\item Did you include complete proofs of all theoretical results?
    \answerYes{}
\end{enumerate}

\item If you ran experiments...
\begin{enumerate}
  \item Did you include the code, data, and instructions needed to reproduce the main experimental results (either in the supplemental material or as a URL)?
    \answerNA{}
  \item Did you specify all the training details (e.g., data splits, hyperparameters, how they were chosen)?
    \answerNA{}
	\item Did you report error bars (e.g., with respect to the random seed after running experiments multiple times)?
    \answerNA{}
	\item Did you include the total amount of compute and the type of resources used (e.g., type of GPUs, internal cluster, or cloud provider)?
    \answerNA{}
\end{enumerate}

\item If you are using existing assets (e.g., code, data, models) or curating/releasing new assets...
\begin{enumerate}
  \item If your work uses existing assets, did you cite the creators?
    \answerNA{}
  \item Did you mention the license of the assets?
    \answerNA{}
  \item Did you include any new assets either in the supplemental material or as a URL?
    \answerNA{}
  \item Did you discuss whether and how consent was obtained from people whose data you're using/curating?
    \answerNA{}
  \item Did you discuss whether the data you are using/curating contains personally identifiable information or offensive content?
    \answerNA{}
\end{enumerate}

\item If you used crowdsourcing or conducted research with human subjects...
\begin{enumerate}
  \item Did you include the full text of instructions given to participants and screenshots, if applicable?
    \answerNA{}
  \item Did you describe any potential participant risks, with links to Institutional Review Board (IRB) approvals, if applicable?
    \answerNA{}
  \item Did you include the estimated hourly wage paid to participants and the total amount spent on participant compensation?
    \answerNA{}
\end{enumerate}

\end{enumerate}

\newpage
\appendix

\section{More Discussions on the Regret and Sample Complexity}\label{section:discussion}
\subsection{Converting
 Sample Complexity of Exploration to Regret}\label{app:conversion}
In this subsection, we shows the relationship between the sample complexity of exploration and the regret. 

The definition of regret in Defintion \ref{def:regret} is related to the ``sample complexity of exploration'' $N(\epsilon,\delta)$ \citep{kakade2003sample,lattimore2012pac,dong2019q}, which is the upper bound on the number of steps $t$ such that $\vvalue^*(s_t) - \vvalue^{\pi}_t(s_t)\ge \epsilon$ with probability at least $1-\delta$. Compared with the regret, sample complexity of exploration focuses on the sub-optimalities at all steps $t$, rather than the first $T$ steps, and ignores the small sub-optimalities. Though both  metrics have been used to describe the performance of an algorithm, these two metrics are not directly comparable. More specifically, algorithms with fewer but larger sub-optimalities will have a small sample complexity of exploration but a high regret. In contrast, algorithms with a lot of moderate sub-optimalities will have a high sample complexity of exploration but a low regret.

By the definition of the sample complexity exploration $N(\epsilon,\delta)$, with probability at least $1-\delta$, the number of steps $t$  where  $\vvalue^*(s_t) - \vvalue^{\pi}_t(s_t)\ge \epsilon$ is upper bounded by $N(\epsilon,\delta)$. Thus, for the regret within $T$ steps, we have following inequality:
\begin{align}
    \text{Regret}(T) &=  \sum_{t=1}^T \big[\vvalue^*(s_t) - \vvalue^\pi_t(s_t)\big]\notag \\
    &= \sum_{t \in [T],\vvalue^*(s_t) - \vvalue^{\pi}_t(s_t)\ge \epsilon }\big[\vvalue^*(s_t) - \vvalue^\pi_t(s_t)\big]\notag\\
    &\qquad + \sum_{t \in [T],\vvalue^*(s_t) - \vvalue^{\pi}_t(s_t)< \epsilon }\big[\vvalue^*(s_t) - \vvalue^\pi_t(s_t)\big] \notag \\
    & \leq \frac{N(\epsilon,\delta)}{1-\gamma}+ T\epsilon,\label{eq:sample-complexity-regret} 
\end{align}
where the inequality holds due to the definition of $N(\epsilon,\delta)$. Furthermore,if an algorithm achieve sample complexity $N(\epsilon,\delta)=O(B\epsilon^{-\alpha})$, then we can choose $\epsilon=T^{-1/(\alpha+1)}(1-\gamma)^{1/(\alpha+1)}B^{-1/(\alpha+1)}$ to minimize the \eqref{eq:sample-complexity-regret}. Thus, we have
\begin{align}
    \text{Regret}(T) &\leq \frac{N(\epsilon,\delta)}{1-\gamma}+ T\epsilon\\\notag
    &=O\Big(\frac{B\epsilon^{-\alpha}}{1-\gamma}+T\epsilon\Big)\notag\\
    &=O\big(B^{1/(\alpha+1)}(1-\gamma)^{-1/(\alpha+1)}T^{\alpha/(\alpha+1)}\big).\notag
\end{align}
Furthermore, the best result in sample complexity of exploration \citep{zhang2020model} achieves $\tilde{O}\Big({SA}/{\big((1-\gamma)^3\epsilon^2\big)}\Big)$ sample complexity and this result implies $\tilde{O}(S^{1/3}A^{1/3}(1-\gamma)^{-4/3}T^{2/3})$ regret, which is worse than our result by a $T^{1/6}$ factor.

\subsection{Comparison with the Regret in \cite{liu2020regret}}

Our definition is similar to that of \citet{liu2020regret}. Note that \citet{liu2020regret} define the regret as $\text{Regret}^{\text{Liu}}(T) = \sum_{t=1}^T \Delta_t$, where $\Delta_t = (1-\gamma)\vvalue^*(s_t) -\reward(s_t,a_t)$.
Comparing the definition in \citet{liu2020regret} with our definition, we can show that $(1-\gamma)\text{Regret}(T) \approx \text{Regret}^{\text{Liu}}(T)$ since
\begin{align}
   (1-\gamma)\sum_{t=1}^T  \vvalue^{\pi}_t(s_t)  &\approx (1-\gamma)\sum_{t=1}^T\sum_{i=0}^{\infty} \gamma^i \reward(s_{t+i},a_{t+i})
   \approx \sum_{t=1}^T  \reward(s_t,a_t)\notag,
\end{align}
where the first approximate equality holds due to  Azuma-Hoeffding inequality and the second approximate equality holds due to $0\leq \reward(s, a)\leq 1$. Therefore, our regret definition is equivalent to that in \cite{liu2020regret} up to a $1-\gamma$ factor.

\section{Proof of Lemmas in Section \ref{section: main}}\label{sec:aaa}

In this section, we prove Lemma \ref{lemma:UCB} to Lemma \ref{theorem: I_5}. For simplicity, we introduce the following shorthand notations:
\begin{align}
    &\VV^*(s,a)=\text{Var}_{s'\sim \PP(\cdot|s,a)}\big(\vvalue^*(s')\big),\notag\\
    &\VV^{\pi}_t(s,a)=\text{Var}_{s'\sim \PP(\cdot|s,a)}\big(\vvalue^{\pi}_{t+1}(s')\big),\notag\\
    &\VV_t(s,a)=\text{Var}_{s'\sim \PP_t(\cdot|s,a)}(V_{t}(s')),\notag\\
     &\VV_t^*(s,a)=\text{Var}_{s'\sim \PP_t(\cdot|s,a)}(V^*(s')).\notag
\end{align}
We start with a list of technical lemmas that will be used to prove Lemma \ref{lemma:UCB} to Lemma \ref{theorem: I_5}. We first provide the Azuma-Hoeffding and Bernstein inequalities. 

\begin{lemma}[Azuma–Hoeffding inequality, \citealt{cesa2006prediction}]\label{lemma:azuma}
Let $\{x_i\}_{i=1}^n$ be a martingale difference sequence with respect to a filtration $\{\cG_{i}\}$ satisfying $|x_i| \leq M$ for some constant $M$, $x_i$ is $\cG_{i+1}$-measurable, $\EE[x_i|\cG_i] = 0$. Then for any $0<\delta<1$, with probability at least $1-\delta$, we have 
\begin{align}
    \sum_{i=1}^n x_i\leq M\sqrt{2n \log (1/\delta)}.\notag
\end{align} 
\end{lemma}

\begin{lemma}[Bernstein inequality, \citealt{cesa2006prediction}]\label{lemma:freedman}
Let $\{x_i\}_{i=1}^n$ be a martingale difference sequence with respect to a filtration $\{\cG_{i}\}$ satisfying $|x_i| \leq M$ for some constant $M$, $x_i$ is $\cG_{i+1}$-measurable, $\EE[x_i|\cG_i] = 0$. Suppose that
\begin{align}
    \sum_{i=1}^n \EE(x_i^2|\cG_{i}) \leq v\notag
\end{align}
for some constant $v$. Then for any $\delta >0$, with probability at least $1-\delta$, 
\begin{align}
    \sum_{i=1}^n x_i \leq \sqrt{2v \log(1/\delta)} + \frac{2M \log(1/\delta)}{3}.\notag
\end{align}
\end{lemma}

The following first lemma provides basic inequalities for the summations of counted numbers $N_i(s_i, a_i)$ and $N_i(s_i)$. 
\begin{lemma}\label{lemma:sum-of-state}
For all $t \in [T]$ and subset $\cC \subseteq   [T]$, we have
\begin{align}
    &\sum_{i=1}^t\frac{1}{N_{\CC{i-1}}(s_i,a_i)\vee 1}\leq SA\CC{\log(3T)},\notag\\
    &\sum_{i=1}^t\frac{1}{N_{i-1}(s_i)\vee 1}\leq S\CC{\log(3T)},\notag\\
    &\sum_{i\in \cC}\frac{1}{\sqrt{N_{i-1}(s_i,a_i)\vee 1}}\leq \sqrt{SA\CC{\log(3T)}|\cC|}.\notag
\end{align}
\end{lemma}
Next lemma upper bounds the difference between the empirical measure $\PP_{t-1}$ and $\PP$, with respect to the true variance of the optimal value function $\VV^*(s,a)$.

\begin{lemma}\label{lemma:P-difference}
If  $0\leq \vvalue^*(s)\leq 1/(1-\gamma)$ for all $s\in \cS$, then with probability at least $1-\delta$, for all $t\in [T], s\in \cS, a \in \cA$,  we have
\begin{align}
    \big[(\PP_{t}-\PP)\vvalue^*\big](s,a)\leq \sqrt{\frac{2\VV^*(s,a)\CC{\log(SAT/\delta)}}{N_{t-1}(s,a)\vee 1}}+\frac{\CC{2\log(SAT/\delta)}}{3(1-\gamma)\big(N_{t-1}(s,a)\vee 1\big)}.\notag
\end{align}
\end{lemma}

Similar to Lemma \ref{lemma:P-difference}, the following lemmas also upper bounds the difference between the empirical measure $\PP_{t-1}$ and $\PP$, but with respect to the estimated variance. 

\begin{lemma}[Theorem 4 in \citealt{maurer2009empirical}]\label{lemma: maurer-3}
Let $Z,Z_1,..,Z_n$ be i.i.d random variable with value in $[0,M]$ and let $\delta>0$, then with probability at least $1-\delta$, we have
\begin{align}
    \EE Z-\frac{1}{n}\sum_{i=1}^n Z_i\leq \sqrt{\frac{2\VV_n Z \log(1/\delta)}{n}}+\frac{7M\log(1/\delta)}{3n},\notag
\end{align}
where $\VV_n Z$ is the estimated variance $\VV_n Z=\sum_{1\leq i<j\leq n} (Z_i-Z_j)^2/n(n-1)$.
\end{lemma}

\begin{lemma}\label{lemma:P-estimate-difference}
If  $0\leq \vvalue^*(s)\leq 1/(1-\gamma)$ for all $s\in \cS$, then with probability at least $1-\delta$, for all $t\in [T], s\in \cS, a \in \cA$, we have
\begin{align}
    \big[(\PP-\PP_{t})\vvalue^*\big](s,a)\leq \sqrt{\frac{2\VV_{t-1}^*(s,a)\CC{\log(SAT/\delta)}}{N_{t-1}(s,a)\vee 1}}+\frac{7\CC{\log(SAT/\delta)}}{3(1-\gamma)\big(N_{t-1}(s,a)\vee 1\big)}.\notag
\end{align}
\end{lemma}

The next lemma shows that the total variance of the nonstationary policy $\pi$ can be upper bounded by $O(T/(1-\gamma))$. It is worth noting that a trivial bound which bounds $\VV^{\pi}_i(s_i,a_i)$ by $1/(1-\gamma)^2$ only gives an $O(T/(1-\gamma)^2)$ bound. 
\begin{lemma}\label{lemma:total-variance}
With probability at least $1-\delta/(1-\gamma)$, we have
\begin{align}
    \gamma^2\sum_{t=1}^T  \VV^{\pi}_t(s_{t},a_{t}) \leq \frac{5T}{1-\gamma} + \frac{25\log(1/\delta)}{3(1-\gamma)^3}.\notag
\end{align}
\end{lemma}

Based on previous concentration Lemma, we define the following high probability events and our proof of Lemma \ref{theorem: I_1} to Lemma \ref{theorem: I_5}  relies on these high probability events.
Let $\cE$ denote the event when the conclusion of Lemma \ref{lemma:UCB} holds. Then by Lemma \ref{lemma:UCB}, we have $\Pr(\cE)\ge 1-64T\delta \log^2 T/(1-\gamma)^2$. We also define the following event:
\begin{align}
 \cE_1&=\bigg\{\big[(\PP_{t}-\PP)\vvalue^*\big](s,a)\leq \sqrt{\frac{2\VV^*(s,a)\CC{\log(SAT/\delta)}}{N_{t-1}(s,a)\vee 1}}\notag\\
 &\qquad\qquad\qquad\qquad\qquad\qquad+\frac{\CC{2\log(SAT/\delta)}}{3(1-\gamma)\big(N_{t-1}(s,a)\vee 1\big)},\forall s\in \cS, a \in \cA,  t\in [T]\bigg\},\notag\\
 \cE_2&=\bigg\{\big[(\PP-\PP_{t})\vvalue^*\big](s,a)\leq \sqrt{\frac{2\VV_{t-1}^*(s,a)\CC{\log(SAT/\delta)}}{N_{t-1}(s,a)\vee 1}}\notag\\
 &\qquad\qquad\qquad\qquad\qquad\qquad +\frac{7\CC{\log(SAT/\delta)}}{3(1-\gamma)\big(N_{t-1}(s,a)\vee 1\big)},\forall s\in \cS, a \in \cA,  t\in [T]\bigg\}\notag,\\
 \cE_3&=\bigg\{\PP_{t-1}(s'|s_t,a_t)-\PP(s'|s_t,a_t)\leq \sqrt{\frac{2\PP(s'|s_t,a_t)(1-\PP(s'|s_t,a_t))\CC{\log(ST/\delta)}}{N_{t-1}(s_t,a_t)\vee 1}},\notag\\
 &\qquad+\frac{\CC{2\log(ST/\delta)}}{3\big(N_{t-1}(s_t,a_t)\vee 1\big)} \forall s\in \cS, a \in \cA,  t\in [T]\bigg\}\notag,\\
 \cE_4&=\bigg\{\sum_{t=1}^{T}  \PP(s'|s_t,a_t) \big(\vvalue_{t-1}(s')-\vvalue^*(s')\big)\leq {\sum_{t=1}^{T}\big(\vvalue_{t-1}(s_{t+1})-\vvalue^*(s_{t+1})\big)} +\frac{\sqrt{2T\log(1/\delta)}}{1-\gamma} \bigg\}\notag,\\
 \cE_5&=\big\{\gamma^2\sum_{t=1}^T  \VV^{\pi}_t(s_{t},a_{t}) \leq \frac{5T}{1-\gamma} + \frac{25\log(1/\delta)}{3(1-\gamma)^3}\big\}\notag,\\
 \cE_6&=\bigg\{\sum_{t=1}^{T}\big[\PP(\vvalue_{t-1}-\vvalue^{\pi}_{t+1})\big](s_t,a_t)-\sum_{t=1}^{T}\big[\vvalue_{t-1}(s_{t+1})-\vvalue^{\pi}_{t+1}(s_{t+1})\big]\leq \frac{\sqrt{2T\log(1/\delta)}}{1-\gamma}\bigg\}\notag,\\
  \cE_7&=\bigg\{\sum_{t=1}^{T}\big[\PP(\vvalue^*-\vvalue^{\pi}_{t+1})\big](s_t,a_t)-\sum_{t=1}^{T}\big[\vvalue^*(s_{t+1})-\vvalue^{\pi}_{t+1}(s_{t+1})\big]\leq \frac{\sqrt{2T\log(1/\delta)}}{1-\gamma}\bigg\}\notag,\\
  \cE_8&=\bigg\{\big\|\PP_{t-1}(\cdot|s,a)-\PP(\cdot|s,a)\big\|_1\leq \frac{\sqrt{2S\CC{\log(T/\delta)}}}{\sqrt{N_{t-1}(s,a)\vee 1}},\forall s\in \cS,a\in \cA,t\in [T]\bigg\}\notag,\\
  \cE_9&=\bigg\{\sum_{t=1}^{T}\sum_{s'}\PP(s'|s_t,a_t)\min\Big\{\frac{100S^2A^2U^5}{(1-\gamma)^5\big(N_{t-1}(s')\vee 1\big)},\frac{1}{(1-\gamma)^2}\Big\}\notag\\
  &\qquad\leq \sum_{t=1}^T \min\Big\{\frac{100S^2A^2U^5}{(1-\gamma)^5\big(N_{t-1}(s_{t+1})\vee 1\big)},\frac{1}{(1-\gamma)^2}\Big\}+\frac{\sqrt{2TU}}{(1-\gamma)^2}\bigg\}\notag,
  \end{align}
where $ U=\log ({40SAT^3\log^2 T}/{(\delta}(1-\gamma)^2)).$
For these high probability events, according to the Lemma \ref{lemma:azuma}, we have $\Pr(\cE_4)\ge 1-\delta,\Pr(\cE_6)\ge 1-\delta,\Pr(\cE_7)\ge 1-\delta,\Pr(\cE_8)\ge 1-\delta,\Pr(\cE_9)\ge 1-\delta.$ According to the Lemma \ref{lemma:freedman}, we have $\Pr(\cE_3)\ge 1-\delta$. According to the Lemma \ref{lemma:P-difference}, we have $\Pr(\cE_1)\ge 1-\delta$. According to the Lemma \ref{lemma:P-estimate-difference}, we have $\Pr(\cE_2)\ge 1-\delta$. According to the Lemma \ref{lemma:total-variance}, we have $\Pr(\cE_5)\ge 1-\delta/(1-\gamma).$

The next lemma shows that the total difference between the optimal variance and the variance induced by $\pi$ can be bounded in terms of $\text{Regret}'(T)$. 
\begin{lemma}\label{lemma:opt-variance-gap}
On the event $\cE_7$,  we have 
\begin{align}
    \sum_{i=1}^T\big(\VV^*(s_i,a_i)-\VV^{\pi}_i(s_i,a_i)\big)\leq \frac{2\text{Regret}'(T)}{1-\gamma}+\frac{2+\sqrt{2T\CC{\log(1/\delta)}}}{(1-\gamma)^2}.\notag
\end{align}
\end{lemma}
Similar to Lemma \ref{lemma:opt-variance-gap}, the next lemma shows that the total difference between the estimated variance and the variance induced by $\pi$ can be upper-bounded in terms of $\text{Regret}'(T)$. 
\begin{lemma}\label{lemma:estimate-variance-gap}
On the event $\cE_6\cap \cE_8$,  we have 
\begin{align}
   \sum_{i=1}^T\big(\VV_{i-1}(s_i,a_i)-\VV^{\pi}_i(s_i,a_i)\big)\leq \frac{2\text{Regret}'(T)}{1-\gamma}+\frac{9S\sqrt{2AT\CC{\log(T/\delta)\log(3T)}}}{(1-\gamma)^2}.\notag
\end{align}
\end{lemma}

\subsection{Proof of Lemma \ref{lemma:UCB}}\label{state-regret}
For simplicity, we denote $U=\log (SAT^2/\delta)$ and $H=\lfloor 2\log T/(1-\gamma)\rfloor+1$ and for $h\in[H]$, we define 
\begin{align}
    \text{Regret}'(t,s,h)= \sum_{1\leq i\leq t,s_i=s} \gamma^h\big[\vvalue_{i+h}(s_{i+h}) - \vvalue^{\pi}_{i+h}(s_{i+h})\big]\notag.
\end{align}
 Then we have the following lemma.

\begin{lemma}\label{lemma:regret-state}
 For each $t\in[T]$, with probability at least $1-4H^2\delta$, for all $s\in \cS,h\in[H]$, we 
have
\begin{align}
    \text{Regret}'(t,s,h)\leq \frac{16SAU^2\sqrt{N_t(s)}}{(1-\gamma)^{2.5}}+\frac{4S^2A^{1.5}U^3}{(1-\gamma)^{3.5}}.\notag
\end{align}
In addition, if $N_t(s)> 0$, we have 
\begin{align*}
    \vvalue_t(s)-\vvalue^*(s)\leq\frac {20SAU^2}{(1-\gamma)^{2.5}\sqrt{N_t(s)}}.
\end{align*}
\end{lemma}

Now, we start the proof of Lemma \ref{lemma:UCB},
\begin{proof}[Proof of Lemma \ref{lemma:UCB}]

\noindent We prove this lemma by induction. At the first step $t=1$, for all $s\in \cS$, we have $\vvalue_1(s)=1/(1-\gamma)\ge \vvalue^*(s)$. When Lemma \ref{lemma:UCB} holds for the first $t$ steps, we consider for each $s\in\cS, a\in\cA$, then by the update rule \eqref{eq:update}, we have
\begin{align}
    \qvalue_{t+1}(s,a)=\min\Big\{\qvalue_{t}(s,a),\reward(s,a)+\gamma [\PP_t\vvalue_{t}](s,a)+\CC{\gamma}\text{UCB}_t(s,a)\Big\}.\notag
\end{align}
If $\qvalue_{t+1}(s,a)=\qvalue_{t}(s,a)$, then by induction, we have
\begin{align}
    \qvalue_{t+1}(s,a)\ge \reward(s,a)+\frac{8\gamma U}{1-\gamma} \ge \reward(s,a)+ \gamma[ \PP\vvalue^*](s,a)=\qvalue^*(s,a),\notag
\end{align}
where the first inequality holds due to \eqref{eq:update} in Algorithm \ref{algorithm} and the second inequality holds due to $0\leq \vvalue^*(s) \leq 1/(1-\gamma)$.
Otherwise,  if $N_{t}(s,a)=0$, then we have
\begin{align}
    \qvalue_{t+1}(s,a)=\qvalue_{t}(s,a)\ge \qvalue^{*}(s,a).\notag
\end{align}
When $N_{t}(s,a)> 0$, with probability at least $1-\delta$, we have
\begin{align}
    &\qvalue_{t+1}(s,a)-\qvalue^{*}(s,a)\notag \\
    &=\gamma [\PP_t\vvalue_{t}](s,a)+\CC{\gamma}\text{UCB}_t(s,a)-\gamma[\PP\vvalue^*](s,a)\notag\\
    &=\CC{\gamma}\text{UCB}_t(s,a)+\gamma[(\PP_t-\PP)\vvalue^*](s,a)+\gamma [\PP_t(\vvalue_{t}-\vvalue^*)](s,a)\notag\\
    &\ge \CC{\gamma}\text{UCB}_t(s,a)+\gamma[(\PP_t-\PP)\vvalue^*](s,a)\notag\\
    &\ge \CC{\gamma}\text{UCB}_t(s,a)-\CC{\gamma} \sqrt{\frac{4\VV_t^*(s,a)U}{N_t(s,a)\vee 1}}-\frac{8U\CC{\gamma}}{(1-\gamma)\big(N_t(s,a)\vee 1\big)}\notag\\
    &\ge \CC{\gamma}\sqrt{\frac{8\VV_t(s,a)U}{N_t(s,a)\vee 1}}-\CC{\gamma}\sqrt{\frac{4\VV_t^*(s,a)U}{N_t(s,a)\vee 1}}  +\CC{\gamma}\sqrt{\frac{8\sum_{s'}\PP_t(s'|s,a)\min\big\{100B_t(s'),{1}{/(1-\gamma)^2}\big\}}{N_t(s,a)\vee 1}},\label{eq:17}
\end{align}
where the first inequality holds due to $\vvalue_t(s)\ge\vvalue^*(s)$ , the second inequality holds due to Lemma~\ref{lemma:P-estimate-difference} and the third inequality holds due to the definition of $\text{UCB}_t$ in \eqref{eq:UCB1}. For the term $\VV_t^*(s,a)$, we have
\begin{align}
    \VV_t^*(s,a)&=\EE_{s'\sim \PP_t(\cdot|s,a)}\bigg[\big(\vvalue^*(s')-\EE[\vvalue
    ^*(s')]\big)^2\bigg]\notag\\
    &=\EE_{s'\sim \PP_t(\cdot|s,a)}\bigg[\big(\vvalue^*(s')-\vvalue_t(s')-\EE[\vvalue^*(s')-\vvalue_t(s')]+\vvalue_t(s')-\EE[\vvalue_t(s')]\big)^2\bigg]\notag\\
    &\leq 2\EE_{s'\sim \PP_t(\cdot|s,a)}\bigg[\big(\vvalue_t(s')-\EE[\vvalue_t(s')]\big)^2\bigg]\notag\\
    &\qquad +2\EE_{s'\sim \PP_t(\cdot|s,a)}\bigg[\big(\vvalue^*(s')-\vvalue_t(s')-\EE[\vvalue^*(s')-\vvalue_t(s')]\big)^2\bigg]\notag\\
    &\leq 2\VV_t(s,a)+2\EE_{s'\sim \PP_t(\cdot|s,a)}\bigg[\big(\vvalue^*(s')-\vvalue_t(s')\big)^2\bigg],\label{eq:12}
\end{align}
where the first inequality holds due to $(x+y)^2 \leq 2x^2+2y^2$ and the second inequality holds due to $\EE\big[(X-\EE[X])^2\big]\leq \EE[X^2]$. Substituting \eqref{eq:12} into \eqref{eq:17}, with probability at least $1-4(t+1)H^2\delta$, we have
\begin{align}
    \qvalue_{t+1}(s,a)-\qvalue^{*}(s,a)
    & \ge \CC{\gamma}\sqrt{\frac{8\VV_t(s,a)U}{N_t(s,a)\vee 1}} +\CC{\gamma}\sqrt{\frac{8\sum_{s'}\PP_t(s'|s,a)\min\big\{100B_t(s'),{1}{/(1-\gamma)^2}\big\}}{N_t(s,a)\vee 1}}\notag\\
    &\qquad - \CC{\gamma}\sqrt{\frac{8\VV_t(s,a)U+8U\EE_{s'\sim \PP_t(\cdot|s,a)}\big(\vvalue^*(s')-\vvalue_t(s')\big)^2}{N_t(s,a)\vee 1}}\notag\\
    &\ge \CC{\gamma}\sqrt{\frac{8\sum_{s'}\PP_t(s'|s,a)\min\big\{100B_t(s'),{1}{/(1-\gamma)^2}\big\}}{N_t(s,a)\vee 1}}\notag\\
    &\qquad -\CC{\gamma}\sqrt{\frac{8U\EE_{s'\sim \PP_t(\cdot|s,a)}\big(\vvalue^*(s')-\vvalue_t(s')\big)^2}{N_t(s,a)\vee 1}}\notag\\
    &\ge 0,\notag
\end{align}
where the first inequality holds due to \eqref{eq:17}, the second inequality holds due to \eqref{eq:12}, the third inequality holds due to $\sqrt{a+b}\leq \sqrt{a}+\sqrt{b}$, the last inequality holds due to Lemma \ref{lemma:regret-state} with probability at least $1-4H^2\delta$ and induction hypothesis with probability at least $1-4tH^2\delta$. In addition, for all $s\in \cS$, we have
\begin{align}
    \vvalue_{t+1}(s)=\max_{a\in \cA} \qvalue_{t+1}(s,a)\ge \max_{a\in \cA} \qvalue^*(s,a)=\vvalue^*(s).\notag
\end{align}
Thus, by induction, we complete the proof of Lemma \ref{lemma:UCB}.
\end{proof}

\subsection{Proof of Lemma \ref{theorem: I_1}}
\begin{proof}[Proof of Lemma \ref{theorem: I_1}]

\noindent We have
\begin{align}
&\sum_{t=1}^{T}\gamma\big(\vvalue_{t-1}(s_{t+1})-\vvalue^{\pi}_{t+1}(s_{t+1})\big)\notag\\
&\qquad=\underbrace{\gamma \sum_{t=1}^{T}\big(\vvalue_{t-1}(s_{t+1})-\vvalue_{t+1}(s_{t+1})\big)}_{I_1}+\underbrace{\gamma \sum_{t=1}^{T}\big(\vvalue_{t+1}(s_{t+1})-\vvalue^{\pi}_{t+1}(s_{t+1})\big)}_{I_2}.\notag
\end{align}
For the term $I_1$, we have 
\begin{align}
    \sum_{t=1}^{T}\gamma\big(\vvalue_{t-1}(s_{t+1})-\vvalue_{t+1}(s_{t+1})\big)
    & \leq \gamma\sum_{t=1}^{T} \sum_{s\in \cS}\big[\vvalue_{t-1}(s)-\vvalue_{t+1}(s)\big] \notag\\
    &=\gamma\sum_{s\in \cS}\sum_{t=1}^{T} \big[\vvalue_{t-1}(s)-\vvalue_{t+1}(s)\big]\notag\\
    &=\gamma\sum_{s\in \cS}\big (\vvalue_0(s)+\vvalue_1(s)-\vvalue_{T}(s)-\vvalue_{T+1}(s)\big)\notag\\
    &\leq \frac{2S\gamma}{1-\gamma},\label{eq:I_1}
\end{align}
where the first inequality holds due to $\vvalue_{t-1}(s)\ge \vvalue_{t+1}(s)$ by \eqref{eq:update} in Algorithm \ref{algorithm}, and the second inequality holds due to $0\leq \vvalue_t(s)\leq 1/(1-\gamma)$. 
For the term $I_2$, we have
\begin{align}
    I_2&=\gamma\sum_{t=2}^{T+1}\big(\vvalue_{t}(s_{t})-\vvalue^{\pi}_{t}(s_{t})\big)\notag\\
    &=\gamma\text{Regret}'(T)+\gamma\big(\vvalue_{T+1}(s_{T+1})-\vvalue^{\pi}_{T+1}(s_{T+1})\big)-\gamma\big(\vvalue_{1}(s_{1})-\vvalue^{\pi}_{1}(s_{1})\big)\notag\\
    &\leq \gamma\text{Regret}'(T)+\frac{2\gamma}{1-\gamma},\label{eq:I_4}
\end{align}
where the inequality holds due to $0
\leq \vvalue_t(s),\vvalue^{\pi}_t(s)\leq 1/(1-\gamma)$. 
Combining \eqref{eq:I_1} and \eqref{eq:I_4}, we complete the proof of Lemma \ref{theorem: I_1}.
\end{proof}

\subsection{Proof of Lemma \ref{theorem: J_1}}
\begin{proof}[Proof of Lemma \ref{theorem: J_1}]

\noindent On the event $\cE$, we have
\begin{align}
    &\sum_{t=1}^{T}\gamma\big[(\PP_{t-1}-\PP)(\vvalue_{t-1} -\vvalue^*)\big] (s_t,a_t)\notag \\
    &\qquad=\gamma\sum_{t=1}^{T} \sum_{s'\in \cS}\big(\PP_{t-1}(s'|s_t,a_t)-\PP(s'|s_t,a_t)\big)\big(\vvalue_{t-1}(s')-\vvalue^*(s'))\notag\\
    &\qquad\leq \sum_{t=1}^{T} \sum_{s'\in \cS} \Bigg[\sqrt{\frac{2\PP(s'|s_t,a_t)(1-\PP(s'|s_t,a_t))\CC{\log(2ST/\delta)}}{N_{t-1}(s_t,a_t)\vee 1}}+\frac{\CC{2\log(ST/\delta)}}{3\big(N_{t-1}(s_t,a_t)\vee 1\big)}\Bigg]
    \notag\\
    &\qquad \qquad \times\big(\vvalue_{t-1}(s')-\vvalue^*(s')\big)\notag\\
    &\qquad\leq  \underbrace{\sum_{t=1}^{T}  \sum_{s'\in \cS} \sqrt{2\CC{\log(ST/\delta)}}\sqrt{\frac{\PP(s'|s_t,a_t)}{N_{t-1}(s_t,a_t)\vee 1}}\big(\vvalue_{t-1}(s')-\vvalue^*(s')\big)}_{I_1}
    \notag\\
    &\qquad \qquad+\underbrace{\sum_{t=1}^{T} \frac{ 2S\CC{\log(ST/\delta)}}{3(1-\gamma)\big(N_{t-1}(s_t,a_t)\vee 1\big)}}_{I_2},\label{eq:J_1}
\end{align}
 where first inequality holds due to the definition of $\cE_2$ and the second inequality holds due to $0 \leq \vvalue_{t+1}(s')-\vvalue^*(s')\leq 1/(1-\gamma)$.
To bound term $I_1$, we separate $\cS$ into two subsets $\cS^1_t \cup \cS^2_t$, where 
\begin{align}
    \cS^1_t = \bigg\{s \in \cS: \PP(s|s_t,a_t)\big(N_{t-1}(s_t, a_t)\vee 1\big)\ge \frac{8\CC{\log(ST/\delta)}}{(1-\gamma)^2}\bigg\},\ \cS^2_t = \cS/\cS^1_t.\notag
\end{align}
Then on the event $\cE_4$, we have
\begin{align}
        I_1&=\sum_{t=1}^{T}  \sum_{s'\in \cS^1_t} \PP(s'|s_t,a_t) \sqrt{2\CC{\log(ST/\delta)}  }\sqrt{\frac{1}{ \PP(s'|s_t,a_t)\big(N_{t-1}(s_t, a_t)\vee 1\big)}}\big(\vvalue_{t-1}(s')-\vvalue^*(s')\big)\notag\\
    &\qquad +\sum_{t=1}^{T}  \sum_{s'\in \cS^2_t}  \frac{\sqrt{2\CC{\log(ST/\delta)}\PP(s'|s_t,a_t)\big(N_{t-1}(s_t, a_t)\vee 1\big)}}{N_{t-1}(s_t, a_t)\vee 1}\big(\vvalue_{t-1}(s')-\vvalue^*(s')\big)\notag\\
    &\leq \sum_{t=1}^{T}  \sum_{s'\in \cS^1_t}(1-\gamma) \PP(s'|s_t,a_t) \big(\vvalue_{t-1}(s')-\vvalue^*(s')\big)/2 \notag\\
    &\qquad +\sum_{t=1}^{T}  \sum_{s'\in \cS^2_t}  \frac{4\CC{\log(ST/\delta)}}{3(1-\gamma)^2\big(N_{t-1}(s_t, a_t)\vee 1\big)}\notag\\
    &\leq \sum_{t=1}^{T}  \sum_{s'\in \cS^1_t}(1-\gamma) \PP(s'|s_t,a_t) \big(\vvalue_{t-1}(s')-\vvalue^*(s')\big)/2 +\frac{4S^2A\CC{\log(ST/\delta)\log(3T)}}{3(1-\gamma)^2}\notag\\
    &\leq \sum_{t=1}^{T} \sum_{s'\in \cS} (1-\gamma) \PP(s'|s_t,a_t) \big(\vvalue_{t-1}(s')-\vvalue^*(s')\big)/2 +\frac{4S^2A\CC{\log(ST/\delta)\log(3T)}}{3(1-\gamma)^2}\notag\\
    &\leq (1-\gamma)/2\cdot\bigg[{\sum_{t=1}^{T}\big(\vvalue_{t-1}(s_{t+1})-\vvalue^*(s_{t+1})\big)}+\frac{\sqrt{2T\log(1/\delta)}}{1-\gamma}\bigg]+\frac{4S^2A\CC{\log(ST/\delta)\log(3T)}}{3(1-\gamma)^2}\notag\\
    &\leq  (1-\gamma)/2\cdot\sum_{t=1}^{T}\big(\vvalue_{t-1}(s_{t+1})-\vvalue^{\pi}_{t+1}(s_{t+1})\big)+\sqrt{2T\CC{\log(1/\delta)}}+\frac{4S^2A\CC{\log(ST/\delta)\log(3T)}}{3(1-\gamma)^2}\notag \\
    &\leq  (1-\gamma)/2\cdot\bigg[\text{Regret}'(T)+\frac{(2S+2)}{1-\gamma}\bigg]+\sqrt{2T\CC{\log(1/\delta)}}+\frac{4S^2A\CC{\log(ST/\delta)\log(3T)}}{3(1-\gamma)^2}
    ,\label{eq:K_1}
    \end{align}
    where the first inequality holds due to separate condition of $\PP(s')$, the second inequality holds due to Lemma \ref{lemma:sum-of-state}, the third inequality holds due to $\vvalue_{t-1}(s')\ge \vvalue^*(s')$, the fourth inequality holds due to the definition of event $\cE_4$, the fifth inequality holds due to $\vvalue^*\ge\vvalue^{\pi}_{t+1}$, and the last inequality holds due to Lemma \ref{theorem: I_1}. 
    For the term $I_2$, according to Lemma \ref{lemma:sum-of-state}, we have
\begin{align}
    I_2\leq \frac{2S^2A\CC{\log(ST/\delta)\log(3T)}}{3(1-\gamma)}.\label{eq:K_2}
\end{align}
Substituting \eqref{eq:K_1},\eqref{eq:K_2} into \eqref{eq:J_1}, we complete the proof of Lemma \ref{theorem: J_1}.
\end{proof}

\subsection{Proof of Lemma \ref{theorem: J_2}}
\begin{proof}[Proof of Lemma \ref{theorem: J_2}]

\noindent On the event $\cE_1\cap \cE_5 \cap \cE_7$, we have
\begin{align}
    &\sum_{t=1}^{T}\gamma[(\PP_{t-1}-\PP)\vvalue^*](s_t,a_t)\notag\\
    &\qquad\leq \sum_{t=1}^T \CC{\gamma} \sqrt{\frac{2\VV^*(s_t,a_t)\CC{\log(SAT/\delta)}}{N_{t-1}(s_t,a_t)\vee 1}}+\frac{2\CC{\log(SAT/\delta)}\CC{\gamma}}{(1-\gamma)\big(N_{t-1}(s_t, a_t)\vee 1\big)}\notag\\
    &\qquad\leq \CC{\gamma}\sqrt{2 \CC{\log(SAT/\delta)}}\sqrt{\sum_{t=1}^T\VV^*(s_t,a_t)} \sqrt{\sum_{t=1}^T \frac{1}{N_{t-1}(s_t,a_t)\vee 1}}\notag\\
    &\qquad \qquad+\sum_{t=1}^T \frac{2\gamma \CC{\log(SAT/\delta)}}{(1-\gamma)\big(N_{t-1}(s_t,a_t)\vee 1\big)}\notag\\
    &\qquad\leq \CC{\gamma}U\sqrt{2 SA}\sqrt{\sum_{t=1}^T\VV^*(s_t,a_t)}+\frac{2\gamma SAU^2}{1-\gamma}\notag\\
    &\qquad= \CC{\gamma}U\sqrt{2 SA}\sqrt{\sum_{t=1}^T\VV^{\pi}_t(s_t,a_t)+\sum_{t=1}^T\VV^{*}(s_t,a_t)-\sum_{t=1}^T\VV^{\pi}_t(s_t,a_t)} +\frac{2\gamma SAU^2}{1-\gamma}\notag\\
    &\qquad\leq U\sqrt{2 SA}\sqrt{\frac{5T}{1-\gamma}+\frac{29U}{3(1-\gamma)^3}+\frac{2\text{Regret}’(T)}{1-\gamma}+\frac{\sqrt{2TU}}{(1-\gamma)^2}}+\frac{2 SAU^2}{1-\gamma},\label{eq:J_2}
\end{align}
where the first inequality holds due to the definition of event $\cE_1$, the second inequality holds due to Cauchy-Schwarz inequality, the third inequality holds due to Lemma \ref{lemma:sum-of-state} and the definition of $U$, and the last inequality holds due to Lemma \ref{lemma:opt-variance-gap} and the definition of event $\cE_5$. Thus, we complete the proof of Lemma \ref{theorem: J_2}.
\end{proof}

\subsection{Proof of Lemma \ref{theorem: I_5}}
\begin{proof}[Proof of Lemma \ref{theorem: I_5}]

\noindent For the term $\text{UCB}_{t-1}(s_t,a_t)$, we have
\begin{align}
   \sum_{t=1}^{T} \CC{\gamma}\text{UCB}_{t-1}(s_t,a_t) &\leq  \underbrace{\sum_{t=1}^{T} 
    \CC{\gamma}\sqrt{\frac{8U\VV_{t-1}(s_t,a_t)}{N_{t-1}(s_t,a_t)\vee 1}}}_{I_1}+\underbrace{\sum_{t=1}^{T}\CC{\gamma}\frac{8U}{(1-\gamma)\big(N_{t-1}(s_t, a_t)\vee 1\big)}}_{I_2}\notag\\
        &\qquad +\underbrace{\sum_{t=1}^{T}\CC{\gamma}\sqrt{\frac{8\sum_{s'}\PP_t(s'|s_t,a_t)\min\big\{100B_t(s'),{1}{/(1-\gamma)^2}\big\}}{N_{t-1}(s_t,a_t)\vee 1}}}_{I_3}.\label{eq:I_5}
\end{align}
For the term $I_1$, on the event $\cE_5\cap \cE_6\cap\cE_8$, we have
\begin{align}
    I_1&\leq  \CC{\gamma}\sqrt{8U\sum_{t=1}^{T} \VV_{t-1}(s_t,a_t)}\sqrt{\sum_{t=1}^{T}\frac{1}{N_{t-1}(s_t,a_t)\vee 1} }\notag\\
    &\leq \CC{\gamma}U\sqrt{8SA} \sqrt{\sum_{t=1}^{T} \VV_{t-1}(s_t,a_t)}\notag\\
    &=\CC{\gamma}U\sqrt{8SA} \sqrt{\sum_{i=1}^T\VV^{\pi}_t(s_t,a_t)+\sum_{t=1}^{T} \VV_{t-1}(s_t,a_t)-\sum_{i=1}^T\VV^{\pi}_t(s_t,a_t)}\notag\\
    &\leq U\sqrt{8SA}\sqrt{\frac{5T}{1-\gamma}+\frac{29U}{3(1-\gamma)^3}+\frac{2\text{Regret}’(T)}{1-\gamma}+\frac{9SU\sqrt{AT}}{(1-\gamma)^2}},\label{eq:L_1}
\end{align}
where the first inequality holds due to Cauchy-Schwarz inequality, the second inequality holds due to Lemma \ref{lemma:sum-of-state}, the last inequality holds due to the definition of event $\cE_5$ and Lemma \ref{lemma:estimate-variance-gap}. For the term $I_2$, by Lemma \ref{lemma:sum-of-state}, we have
\begin{align}
    I_2 = \sum_{t=1}^{T}\frac{8U}{(1-\gamma)\big(N_{t-1}(s_t, a_t)\vee 1\big)}\leq \frac{8 SAU^2}{1-\gamma}.\label{eq:L_2}
\end{align}
For the term $I_3$, on the event $\cE_8\cap \cE_9$, we have
\begin{align}
    &I_3\notag\\
    &\leq \sqrt{8\sum_{t=1}^{T}\frac{1}{N_{t-1}(s_t,a_t)\vee 1}}\sqrt{\sum_{t=1}^{T}\sum_{s'}\PP_t(s'|s_t,a_t)\min\bigg\{\frac{100S^2A^2U^5}{(1-\gamma)^5N_{t-1}(s')},\frac{1}{(1-\gamma)^2}\bigg\}}\notag\\
    &\leq \sqrt{8SAU}\sqrt{\sum_{t=1}^{T}\sum_{s'}\PP_t(s'|s_t,a_t)\min\bigg\{\frac{100S^2A^2U^5}{(1-\gamma)^5\big(N_{t-1}(s')\vee 1\big)},\frac{1}{(1-\gamma)^2}\bigg\}}\notag\\
    &\leq  \sqrt{8SAU}\cdot\notag\\
    & \sqrt{\sum_{i=1}^T\frac{\sqrt{2SU}}{(1-\gamma)^2\sqrt{N_t(s_t,a_t)\vee 1}}+\sum_{t=1}^{T}\sum_{s'}\PP(s'|s_t,a_t)\min\bigg\{\frac{100S^2A^2U^5}{(1-\gamma)^5\big(N_{t-1}(s')\vee 1\big)},\frac{1}{(1-\gamma)^2}\bigg\}}\notag\\
    &\leq \sqrt{8SAU}\sqrt{\frac{SU\sqrt{2AT}}{(1-\gamma)^2}+\frac{\sqrt{2TU}}{(1-\gamma)^2}+\sum_{t=1}^T \min\bigg\{\frac{100S^2A^2U^5}{(1-\gamma)^5\big(N_{t-1}(s_{t+1})\vee 1\big)},\frac{1}{(1-\gamma)^2}\bigg\}}\notag\\
    &\leq \sqrt{8SAU} \sqrt{\frac{SU\sqrt{2AT}}{(1-\gamma)^2}+\frac{\sqrt{2TU}}{(1-\gamma)^2}+\frac{100S^3A^2U^6}{(1-\gamma)^5}},\label{eq:L_3}
\end{align}
where the first inequality holds due to Cauchy-Schwarz inequality, the second inequality holds due to Lemma \ref{lemma:sum-of-state}, the third inequality holds due to the definition of event $\cE_8$, the forth inequality holds due to the definition of event $\cE_9$ and the last inequality holds due to Lemma \ref{lemma:sum-of-state}. Substituting \eqref{eq:L_1}, \eqref{eq:L_2} and \eqref{eq:L_3} into \eqref{eq:I_5}, we complete the proof of Lemma \ref{theorem: I_5}.
\end{proof}

\section{Proof of Lemmas in Section \ref{section: second}}

\subsection{Proof of Lemma \ref{lemma:transition}}
\begin{proof}[Proof of Lemma \ref{lemma:transition}]

\noindent We have
\begin{align}
    \EE^*\bigg[\sum_{t=1}^{T} \vvalue^*(s_t) - \vvalue^{\pi}_t(s_t)\bigg]
&=\EE^*\bigg[\sum_{t=1}^{T} \vvalue^*(s_t) - \sum_{k=0}^{\infty}\gamma^k\reward(s_{t+k},a_{t+k})\bigg]\notag\\
    &=\EE^*\bigg[\sum_{t=1}^{T} \Big(\vvalue^*(s_t) - \sum_{k=0}^{t}\gamma^k\reward(s_{t},a_{t})\Big)-\sum_{t=T+1}^{\infty}{\sum_{k=0}^{T}}\gamma^{t-k}\reward(s_t,a_t)\bigg] \notag\\
    &\ge \EE^*\bigg[\sum_{t=1}^{T} \vvalue^*(s_t)-\frac{\reward(s_t,a_t)}{1-\gamma}\bigg]-\sum_{t=T+1}^{\infty}{\sum_{k=0}^{T}}\gamma^{t-k}\notag\\
    &\ge \EE^*\bigg[\sum_{t=1}^{T} \vvalue^*(s_t)-\frac{\reward(s_t,a_t)}{1-\gamma}\bigg]-\frac{4}{(1-\gamma)^2}.
\end{align}
where the first inequality holds due to $0\leq \reward(s_t,a_t)\leq 1$ and the last inequality holds due to $\sum_{k=0}^{\infty} \gamma^k=1/(1-\gamma)$. Thus, we finish the proof of Lemma \ref{lemma:transition}.
\end{proof}

\subsection{Proof of Lemma \ref{lemma:suhao}}
\begin{proof}[Proof of Lemma \ref{lemma:suhao}]
In this proof, we follow the proof technique in \cite{liu2020regret} and \cite{jaksch2010near}.
For simplicity, we denote $\epsilon=\sqrt{A(1-\gamma)/K}/24$ and we first determine the optimal policy in these hard-to-learn MDPs. According to \eqref{eq:bellman}, for optimal policy $\pi^*$, we have
\begin{align}
     \qvalue^*(s,a)
    = \reward(s,a) + \gamma [\PP \vvalue^*](s,a),\notag
\end{align}
For each $j\in [S]$ and state $s=s_{j,1}$, the choice of action $a$ will not effect the reward $\reward(s,a)$ and the probability transition function $\PP(\cdot|s,a)$. For optimal action $a^*$ at state $s=s_{j,0}$, we have
\begin{align}
    \vvalue^*(s_{j,0})
   & = \reward(s,a) + \gamma [\PP \vvalue^*](s,a^*)\notag\\
&=0+\gamma \PP(s_{j,0}|s_{j,0},a^*) \vvalue^*(s_{j,0})+\gamma\PP(s_{j,1}|s_{j,0},a^*) \vvalue^*(s_{j,1}).\notag
\end{align}
Since $\PP(s_{j,0}|s_{j,0},a^*)+\PP(s_{j,1}|s_{j,0},a^*)=1$, we have
\begin{align}
    (1-\gamma) \vvalue^*(s_{j,0}) =\gamma \big(\vvalue^*(s_{j,1})- \vvalue^*(s_{j,0})\big),\notag
\end{align}
and it implies that $\vvalue^*(s_{j,1})\ge \vvalue^*(s_{j,0})$. Therefore, for all action $a\ne a_{j}^* $, we  have $\qvalue^*(s_{j,0},a_{j}^*)\ge \qvalue^*(s_{j,0},a)$ and it further implies that the optimal action at state $s=s_{j,0}$ is $a_{j}^*$. Thus, according to the optimal bellman equation \ref{eq:bellman}, for each $j\in [S]$, we have
\begin{align}
    \vvalue^*(s_{j,0})&=\gamma(1-\gamma+\epsilon)\vvalue^*(s_{j,1})+\gamma(\gamma-\epsilon)\vvalue^*(s_{j,0}),\notag\\
    \vvalue^*(s_{j,1})&=1+\gamma(1-\gamma)\vvalue^*(s_{j+1,1})+\gamma^2\vvalue^*(s_{j,1}),\notag
\end{align}
and it implies that the optimal value function $\vvalue^*$ is 
\begin{align}
    &\vvalue^*(s_{j,0})=\frac{\gamma-\gamma^2+\gamma\epsilon}{(1-\gamma)(1-2\gamma^2+\gamma+\gamma\epsilon)},\notag\\
    &\vvalue^*(s_{j,1})=\frac{1-\gamma^2+\gamma\epsilon}{(1-\gamma)(1-2\gamma^2+\gamma+\gamma\epsilon)}.\notag
\end{align}
\noindent When an agent visits the state set $\{s_{j,0}, s_{j,1}\}$ for the $i$-th time, we denote the state in $\{s_{j,0}, s_{j,1}\}$ it visited as $X_{j,i}$, and the following action selected by the agent as $A_{j,i}$. For each $j\in[S]$, by the definition of $X_{j,i}$, we have
\begin{align}
    &\PP(X_{j,i}=s_{j,1}|X_{j,i-1}=s_{j,0},A_{j,i-1})=1-\gamma+\ind_{A_{j,i}=a^*_j} \epsilon, \notag\\
    &\PP(X_{j,i}=s_{j,0}|X_{j,i-1}=s_{j,0},A_{j,i-1})=\gamma-\ind_{a=a^*_j} \epsilon, \notag\\
    &\PP(X_{j,i}=s_{j,0}|X_{j,i-1}=s_{j,0},A_{j,i-1})=1-\gamma, \notag\\
    &\PP(X_{j,i}=s_{j,1}|X_{j,i-1}=s_{j,1},A_{j,i-1})=\gamma, \notag
\end{align}
where the third equality holds because when $X_{j,i-1}$ leave state $s_{j,0},s_{j,1}$, the next state in $s_{j,0},s_{j,1}$ must be $s_{j,0}$. Similar to the proof of Theorem 5 in \cite{jaksch2010near}, we focus on the first $K$ visits to the state set $\{s_{j,0},s_{j,1}\}$ and let random variable $N_0,N_1$ and $N_0^*$ denote the total number of visit state $s_{j,0}$, the total number of visit state $s_{j,1}$ and the total number of visit state $s_{j,0}$ with action $a_j^*$. By the same argument as the proof of Theorem 5 in \cite{jaksch2010near}, for the random variable $N_1$ and $N_0^*$, we have following property:
\begin{align}
    \EE[N_1]&\leq \frac{K}{2}+ \frac{1}{2(1-\gamma)}+ \frac{\epsilon \EE[N_0^*]}{1-\gamma},\label{eq:1000}
\end{align}
and
\begin{align}
    \EE[N_0^*]&\leq \frac{K}{2A}+\frac{1}{2A(1-\gamma)}+\frac{\epsilon K}{2}\sqrt{\frac{K}{A(1-\gamma)}}+\frac{\epsilon K}{2\sqrt{A}(1-\gamma)}. \label{eq:2000}
\end{align}
Therefore, the regret can be upper bounded by
\begin{align}
   &\EE^*\bigg[\sum_{i=1}^K \vvalue^*(X_{j,i})-\frac{\reward(X_{j,i},A_{j,i})}{1-\gamma}\bigg]\notag\\
   &=\EE [N_0] \big(\vvalue^*(s_{j,0})-0\big)+\EE[N_1]\bigg(\vvalue^*(s_{j,1})-\frac{1}{1-\gamma}\bigg)\notag\\
   &=\frac{(\gamma-\gamma^2+\gamma\epsilon)\big(K-\EE[N_1]\big)-(\gamma-\gamma^2)\EE[N_1]}{(1-\gamma)(1-2\gamma^2+\gamma+\gamma\epsilon)}\notag\\
   &\ge  \frac{\frac{K\gamma\epsilon}{2}-\gamma-\frac{\gamma\epsilon}{2(1-\gamma)}-\frac{\EE[N_0^*]\epsilon(2\gamma-2\gamma^2+\gamma \epsilon)}{1-\gamma}}{(1-\gamma)(1-2\gamma^2+\gamma+\gamma\epsilon)}\notag\\
   &\ge \frac{\frac{K\gamma\epsilon}{2}-\gamma-\frac{\gamma\epsilon}{2(1-\gamma)}-\bigg(\frac{K}{2A}+\frac{1}{2A(1-\gamma)}+\frac{\epsilon K}{2}\sqrt{\frac{K}{A(1-\gamma)}}+\frac{\epsilon K}{2\sqrt{A}(1-\gamma)}\bigg)\frac{\epsilon(2\gamma-2\gamma^2+\gamma \epsilon)}{1-\gamma}}{(1-\gamma)(1-2\gamma^2+\gamma+\gamma\epsilon)}.\label{eq:3000}
\end{align}
where the second inequality holds due to the fact that $\EE[N_0]+\EE[N_1]=K$, the third inequality holds due to \eqref{eq:1000} and the last inequality holds due to \eqref{eq:2000}. Since $K\ge 10SA/(1-\gamma)^4$, $\gamma> 2/3$ and $A\ge 30$, \eqref{eq:3000} can be further bounded by
\begin{align}
    &\EE^*\bigg[\sum_{i=1}^K \vvalue^*(X_{j,i})-\frac{\reward(X_{j,i},A_{j,i})}{1-\gamma}\bigg]\notag\\
    &\ge \frac{\frac{K\gamma\epsilon}{2}-\gamma-\frac{\gamma\epsilon}{2(1-\gamma)}-\bigg(\frac{K}{2A}+\frac{1}{2A(1-\gamma)}+\frac{\epsilon K}{2}\sqrt{\frac{K}{A(1-\gamma)}}+\frac{\epsilon K}{2\sqrt{A}(1-\gamma)}\bigg)\frac{\epsilon(2\gamma-2\gamma^2+\gamma \epsilon)}{1-\gamma}}{(1-\gamma)(1-2\gamma^2+\gamma+\gamma\epsilon)}\notag\\
    &\ge \gamma \times \frac{\frac{K\epsilon}{4}-1-3\epsilon\bigg(\frac{5K}{8A}+\frac{\epsilon K}{2}\sqrt{\frac{K}{A(1-\gamma)}}+\frac{\epsilon K}{2\sqrt{A}(1-\gamma)}\bigg)}{(1-\gamma)(1-2\gamma^2+\gamma+\gamma\epsilon)}\notag\\
    &\ge\gamma \times \frac{\frac{\sqrt{AK(1-\gamma)}}{576}-1}{(1-\gamma)(1-2\gamma^2+\gamma+\gamma\epsilon)}\notag\\
    &\ge \frac{\sqrt{AK}}{2304(1-\gamma)^{1.5}}-\frac{1}{(1-\gamma)^2}, 
\end{align}
where the second inequality holds to  $\epsilon=\sqrt{A(1-\gamma)/K}/24\leq 1-\gamma$ with  $K\ge 10SA/(1-\gamma)^4$, the third inequality holds due to $\epsilon=\sqrt{A(1-\gamma)/K}/24$ with $A\ge 30$ and the last inequality holds due to $\gamma \ge 2/3$ and $\epsilon=\sqrt{A(1-\gamma)/K}/24\leq 1-\gamma$. Therefore, we finish the proof of Lemma \ref{lemma:suhao}.

\end{proof}

\subsection{Proof of Lemma \ref{lemma:remain}}
\begin{proof}[Proof of Lemma \ref{lemma:remain}]

\noindent For each $j\in[S]$ and $t\in[T]$, we denote $H=\lfloor {\log T}/{(1-\gamma)} \rfloor+1$, random variable 
\begin{align}
    Y_{j,i}=\sum_{k=0}^H\gamma^k\reward(X_{j,i+k},A_{j,i+k}),\notag
\end{align}
and filtration $\mathcal{F}_{j,i}$ contain all random variable before $X_{j,i+H}$. For simplicity, we ignore the subscript $j$ and only focus on the subscript $i$.

Since $Y_i$ is $\mathcal{F}_{i}$-measurable and $0\leq Y_i\leq {1}/{(1-\gamma)}$ , for each $k\in[H]$, with probability at least $1-\delta$, we have
\begin{align}
    \sum_{i=\lfloor \frac{K}{H}\rfloor+1}^{\lfloor \frac{t}{H}\rfloor+1} Y_{iH+k}&\leq   \sum_{i=\lfloor \frac{K}{H}\rfloor+1}^{\lfloor \frac{t}{H}\rfloor+1} \EE\bigg[ Y_{iH+k}| \mathcal{F}_{(i-1)H+k}\bigg]+\sqrt{\frac{2t}{1-\gamma}\log \frac{1}{\delta}}\notag\\
    &=\sum_{i=\lfloor \frac{K}{H}\rfloor+1}^{\lfloor \frac{t}{H}\rfloor+1} \vvalue^
    {\pi}_{iH+k}(X_{iH+k})+\sqrt{\frac{2t}{1-\gamma}\log \frac{1}{\delta}} \notag\\
    &\leq \sum_{i=\lfloor \frac{K}{H}\rfloor+1}^{\lfloor \frac{t}{H}\rfloor+1} \vvalue^*(X_{iH+k})+\sqrt{\frac{2t}{1-\gamma}\log \frac{1}{\delta}}, \label{eq:31}
\end{align}
where the first inequality holds due to Lemma \ref{lemma:azuma} and the second inequality holds due to the definition of optimal value function $\vvalue^*$. Taking summation of \eqref{eq:31}, for all $k\in[H]$, with probability at least $1-H\delta$, we have
\begin{align}
    \sum_{i=K+1}^{t} \vvalue^*(X_{i}) + \frac{\sqrt{2t\log \frac{1}{\delta}\log T}}{(1-\gamma)^{1.5}}&\ge\sum_{i=K+1}^{t} Y_i\notag\\
    &=\sum_{i=K+1}^{t} \sum_{k=0}^H\gamma^k\reward(X_{i+k},A_{i+k})\notag\\
    &\ge \sum_{i=K+1}^t r(X_i,A_i) \sum_{k=0}^{\min(H,i-K-1)}\gamma^i\notag\\
    &\ge  \sum_{i=K+1}^t \frac{r(X_i,A_i)}{1-\gamma}- \frac{4}{(1-\gamma)^2},\notag
\end{align}
where the second inequality holds due to $0\leq\reward(s,a)\leq 1.$
Finally, taking union for all $j\in[S]$ and $t\in[T]$, we complete the proof.
\end{proof}

\subsection{Proof of Lemma \ref{lemma:length}}
\begin{proof}[Proof of Lemma \ref{lemma:length}]
\noindent Let $Y_{j,i}$ be an indicator random variables which denote whether the agent at state $X_{j,i}$ with action $A_{j,i}$ goes to the different state. $Y_{j,i} = 1$ if the agent goes to the different state and $Y_{j,i} = 0$ if the agent stay at the same state. Let filtration $\mathcal{F}_{j,i}$ contain all random variables before $X_{j,i}$. Then, for each $j\in[S]$, with probability at least $1-\delta$, we have
\begin{align}
    \sum_{i=1}^K Y_{j,i}&\leq \sum_{i=1}^K \EE\big[Y_{j,i}|\mathcal{F}_{j,i-1}\big]+\sqrt{2K\log \frac{1}{\delta}}\leq (1-\gamma+\epsilon) K+ \sqrt{2K\log \frac{1}{\delta}} \leq 3(1-\gamma)K,\label{eq:ooo0}
\end{align}
where the first inequality holds due to Lemma \ref{lemma:azuma}, the second inequality holds due to the definition of our MDPs and the last one holds due to the selection of $K$. Similarly, with probability at least $1-\delta$, we have
\begin{align}
    \sum_{i=1}^{5K} Y_{j.i}& \ge \sum_{i=1}^{2K} \EE\big[Y_{j,i}|\mathcal{F}_{j,i-1}\big]-\sqrt{10K\log \frac{1}{\delta}}\ge 5K(1-\gamma)-\sqrt{10K\log \frac{1}{\delta}}\ge 4(1-\gamma) K,\label{eq:41}
\end{align}
where the first inequality holds due to Lemma \ref{lemma:azuma}, the second inequality holds due to the definition of our MDPs and the last one holds due to the selection of $K$. Taking a union bound \eqref{eq:ooo0} and \eqref{eq:41} for all $j \in [S]$, then we have \eqref{eq:ooo0} and \eqref{eq:41} hold with probability at least $1-2S\delta$. Let $Z_{j,i}$ be the number of times for the agent to start from state $s_{j,i}$ and travel the next different state in the first $T$ steps. By definition, we have
\begin{align}
    Z_{j,0}+Z_{j,1}=\sum_{i=1}^{T_j}Y_{j,i}.\label{ooo41}
\end{align}
By Pigeonhole principle, there exist a $j^*$ such that $T_{j^*}\ge {T}/{S}=10K>5K$. Therefore, we have
\begin{align}
Z_{j^*,0}+Z_{j^*,1}=\sum_{i=1}^{T_{j^*}}Y_{{j^*},i} \geq \sum_{i=1}^{5K}Y_{{j^*},i} \geq 4(1-\gamma)K.\label{eq:ooo1}
\end{align}
Furthermore, after leaving the state $s_{j^*,0}$, the agent will visit all other states before arrive the state $s_{j^*,0}$ again. Thus, for any $k \in [S]$, the difference between $Z_{{j^*},0}$ and $Z_{k,0}$ is at most 1, so do $Z_{{j^*},1}$ and $Z_{k,1}$. Therefore, for any $k \in [S]$, we have
\begin{align}
    Z_{k,0}+Z_{k,1} \geq Z_{j^*,0}+Z_{j^*,1}-2 \geq 4(1-\gamma)K-2 >  3(1-\gamma)K \geq \sum_{i=1}^K Y_{k,i},\label{eq:ooo2}
\end{align}
where the second inequality holds due to \eqref{eq:ooo1}, the third inequality holds since $K > 2/(1-\gamma)$ and the last one holds due to \eqref{eq:ooo0}. Finally, by \eqref{ooo41} we have $Z_{k,0}+Z_{k,1} = \sum_{i=1}^{T_k}Y_{k,i}$. Combining it with \eqref{eq:ooo2}, we have $\sum_{i=1}^{T_k}Y_{k,i} > \sum_{i=1}^K Y_{k,i}$, which suggests that $T_k > k$. Thus, we complete the proof.

\end{proof}

\section{Proof of Lemmas in Appendix \ref{sec:aaa}}

\subsection{Proof of Lemma \ref{lemma:sum-of-state}}
\begin{proof}[Proof of Lemma \ref{lemma:sum-of-state}]

\noindent We have
\begin{align}
    \sum_{i=1}^t\frac{1}{N_{i-1}(s_i,a_i)\vee 1}&=\sum_{s\in \cS,a\in \cA}1+\sum_{s\in \cS,a\in \cA}\sum_{i=1}^{N_{t-1}(s,a)}\frac {1}{i}\leq SA+\sum_{s\in \cS,a\in \cA}\sum_{i=1}^{t}\frac {1}{i}\leq SA\CC{\log(3T)}.\label{eq:sa}
\end{align}
We also have
\begin{align}
    \sum_{i=1}^t\frac{1}{N_{i-1}(s_i)\vee 1}&=\sum_{s\in \cS}1+\sum_{i=1}^{N_t(s)}\frac {1}{i}\leq S+\sum_{s\in \cS}\sum_{i=1}^{t}\frac {1}{i}\leq S\CC{\log(3T)}.\notag
\end{align}
According to \eqref{eq:sa}, for a subset $\cC \subseteq   [T]$, we have
\begin{align}
  \sum_{i\in \cC}\frac{1}{\sqrt{N_{i-1}(s_i,a_i)\vee 1}}\leq \sqrt{|\cC|\sum_{i\in \cC}\frac{1}{{N_{i-1}(s_i,a_i)\vee 1}} }\leq \sqrt{SA\CC{\log(3T)}|\cC|},\notag
\end{align}
where the first inequality holds due to Cauchy-Schwarz inequality and the second inequality holds due to \eqref{eq:sa}. Thus, we complete the proof.
\end{proof}

\subsection{Proof of Lemma \ref{lemma:P-difference}}
\begin{proof}[Proof of Lemma \ref{lemma:P-difference}]
For each $s\in \cS, a \in \cA,$ we denote $t_0=0$ and 
\begin{align}
    t_i=\min \big\{t|t>t_{i-1}, (s_t,a_t)=(s,a) \big\}.
\end{align}
Here, $t_i$ is the time which state-action pair $(s,a)$ appear for the $i$th time and the random variable $t_i$ is a stopping time. Beside, the random variable $\vvalue^*(s_{t_i+1})(i=1,2.,,)$ are random variable with value in $\big[0,1/(1-\gamma)\big]$ and variance $\VV^*(s,a)$. By Lemma \ref{lemma:freedman} and a union bound, with probability at least $1-\delta$, for all $s\in \cS, a \in \cA, \tau\in [T]$, we have
\begin{align}
    \sum_{i=1}^{\tau}\vvalue^*(s_{t_i+1})-\sum_{i=1}^{\tau}\PP\vvalue^*(s,a)\leq \sqrt{2\tau \VV^*(s,a)\log(SAT/\delta)}+\frac{2\log(SAT/\delta)}{3(1-\gamma)}\notag.
\end{align}
Thus, for all $\tau\in [T]$, we have
\begin{align}
    \big[(\PP_{t_{\tau}+1}-\PP)\vvalue^*\big](s,a)&=\frac{1}{\tau}\sum_{i=1}^{\tau}\vvalue^*(s_{t_i+1})-\frac{1}{\tau}\sum_{i=1}^{\tau}\PP\vvalue^*(s,a)\notag\\
    &\leq \sqrt{\frac{2\VV^*(s,a)\CC{\log(SAT/\delta)}}{\tau}}+\frac{\CC{2\log(SAT/\delta)}}{3(1-\gamma)\tau}\notag\\
    &= \sqrt{\frac{2\VV^*(s,a)\CC{\log(SAT/\delta)}}{N_{t_\tau}(s,a)}}+\frac{\CC{2\log(SAT/\delta)}}{3(1-\gamma)N_{t_\tau}(s,a)}.\label{eq:b1}
\end{align}
In addition, for $\tau=0$, we have
\begin{align}
    \big[(\PP_{t_\tau+1}-\PP)\vvalue^*\big](s,a)&\leq \frac{1}{1-\gamma}\leq \frac{\CC{2\log(SAT/\delta)}}{3(1-\gamma)\big(N_{t_\tau}(s,a)\vee 1\big)},\label{eq:b2}
\end{align}
where the first inequality holds due to $0\leq \vvalue^*(s)\leq 1/(1-\gamma)$ and the second inequality holds due to $N_{t_\tau}(s,a)=0$. Since $\PP_t$ and $N_{t-1}(s,a)$ changed only when $t=t_{\tau}+1$, we complete the proof by combining \eqref{eq:b1} and \eqref{eq:b2}.
\end{proof}

\subsection{Proof of Lemma \ref{lemma:P-estimate-difference}}
\begin{proof}[Proof of Lemma \ref{lemma:P-estimate-difference}]
For each $s\in \cS, a \in \cA,$ we denote $t_0=0$ and denote
\begin{align}
    t_i=\min \big\{t|t>t_{i-1}, (s_t,a_t)=(s,a) \big\}.
\end{align}
Here, $t_i$ is the time which state-action pair $(s,a)$ appear for the $i$th time and the random variable $t_i$ is a stopping time. Beside, the random variable $\vvalue^*(s_{t_i+1})(i=1,2.,,)$ are random variable with value in $\big[0,1/(1-\gamma)\big]$ and variance $\VV^*(s,a)$. By Lemma \ref{lemma: maurer-3} and a union bound, with probability at least $1-\delta$, for all $s\in \cS, a \in \cA, \tau\in [T]$, we have
\begin{align}
    \sum_{i=1}^{\tau}\PP\vvalue_{t}^*(s,a)-\sum_{i=1}^{\tau}\vvalue^*(s_{t_i+1})\leq \sqrt{2\tau \VV_{t_\tau}^*(s,a)\log(SAT/\delta)}+\frac{7\log(SAT/\delta)}{3(1-\gamma)}\notag.
\end{align}
Thus, for all $\tau\in [T]$, we have
\begin{align}
    \big[(\PP-\PP_{t_{\tau}+1})\vvalue^*\big](s,a)&=\frac{1}{\tau}\Big|\sum_{i=1}^{\tau}\vvalue^*(s_{t_i+1})-\sum_{i=1}^{\tau}\PP\vvalue^*(s,a)\Big|\notag\\
    &\leq \sqrt{\frac{2\VV_{t_\tau}^*(s,a)\CC{\log(SAT/\delta)}}{\tau}}+\frac{\CC{7\log(SAT/\delta)}}{3(1-\gamma)\tau}\notag\\
    &= \sqrt{\frac{2\VV_{t_\tau}^*(s,a)\CC{\log(SAT/\delta)}}{N_{t_\tau}(s,a)}}+\frac{\CC{7\log(SAT/\delta)}}{3(1-\gamma)N_{t_\tau}(s,a)}.\label{eq:c1}
\end{align}
In addition, for $\tau=0$, we have
\begin{align}
    \big[(\PP-\PP_{t_\tau+1})\vvalue^*\big](s,a)&\leq \frac{1}{1-\gamma}\leq \frac{\CC{7\log(SAT/\delta)}}{3(1-\gamma)\big(N_{t_\tau}(s,a)\vee 1\big)},\label{eq:c2}
\end{align}
where the first inequality holds due to $0\leq \vvalue^*(s)\leq 1/(1-\gamma)$ and the second inequality holds due to $N_{t_\tau}(s,a)=0$. Since $\PP_t, \VV^*_{t-1}$ and $N_{t-1}(s,a)$ changed only when $t=t_{\tau}+1$, we complete the proof by combining \eqref{eq:c1} and \eqref{eq:c2}.
\end{proof}

\subsection{Proof of Lemma \ref{lemma:total-variance}}
\begin{proof}[Proof of Lemma \ref{lemma:total-variance}]

\noindent For simplicity, we denote $H=\lfloor 1/(1-\gamma)\rfloor+1,T'=\lfloor T/H\rfloor +1$ and filtration $\mathcal{F}_t$ contained all random variables before first $t+H$ steps. Then
for every $t\in[T]$, we have
\begin{align}
    \frac{1}{(1-\gamma)^2}&\ge \EE\bigg[\big(\sum_{i=0}^{\infty}\gamma^i \reward(s_{t+i},a_{t+i})\big)-\vvalue^{\pi}_t(s_t)|\mathcal{F}_{t-H}\bigg]^2\notag\\
    &=\EE\bigg[\sum_{i=0}^{\infty}\gamma^i\big(\reward(s_{t+i},a_{t+i})+\gamma \vvalue^{\pi}_{t+i+1}(s_{t+i+1})-\vvalue^{\pi}_{t+i}(s_{t+i})\big)|\mathcal{F}_{t-H}\bigg]^2\notag\\
    &=\EE \bigg[\sum_{i=0}^{\infty}\gamma^{2i}\bigg[\reward(s_{t+i},a_{t+i})+\gamma \vvalue^{\pi}_{t+i+1}(s_{t+i+1})-\vvalue^{\pi}_{t+i}(s_{t+i})\bigg]^2|\mathcal{F}_{t-H}\bigg]\notag\\
    &=\EE \bigg[ \sum_{i=0}^{\infty}\gamma^{2i+2} \VV^{\pi}_{t+i}(s_{t+i},a_{t+i})|\mathcal{F}_{t-H}\bigg]\notag\\
    &\ge \EE \bigg[ \underbrace{\sum_{i=0}^{H}\gamma^{2i+2} \VV^{\pi}_{t+i}(s_{t+i},a_{t+i})}_{X_t}|\mathcal{F}_{t-H}\bigg],\label{eq:ind-var}
\end{align}
where the first inequality holds due to $0\leq \reward(s,a)\leq 1,0\leq \vvalue^{\pi}_t(s)\leq 1/(1-\gamma)$ and the second inequality holds due to $\VV^{\pi}_{t+i}(s_{t+i},a_{t+i})\ge 0$. For the random variable $X_t$, we have
\begin{align}
   |X_t|&\leq \sum_{i=0}^{H}\frac{\gamma^{2i+2} }{(1-\gamma)^2}\leq \frac{1}{(1-\gamma)^3},\ 
   \text{Var}{\big[|X_t||\mathcal{F}_{t-H}\big]}\leq (\max |X_t|) \EE[X_t|\mathcal{F}_{t-H}]\leq \frac{1}{(1-\gamma)^5},\notag
\end{align}
Since $X_t$ is $\mathcal{F}_{t}$-measurable and $\EE[X_t| \mathcal{F}_{t-H}]\leq 1/(1-\gamma)^2$, for each $i\in[H]$, by Lemma \ref{lemma:freedman}, with probability at least $1-\delta$, we have 
\begin{align}
    \sum_{j=0}^{T'}X_{jH+i}&\leq \sum_{j=0}^{T'}\EE[X_{jH+i}|\mathcal{F}_{(j-1)H+i}]+\sqrt{\frac{2T'\log(1/\delta)}{(1-\gamma)^5}}+ \frac{2\log(1/\delta)}{3(1-\gamma)^3}\notag\\
    &\leq \frac{T'}{(1-\gamma)^2}+\sqrt{\frac{2T'\log(1/\delta)}{(1-\gamma)^5}}+ \frac{2\log(1/\delta)}{3(1-\gamma)^3}.\label{eq:azuma-bern}
\end{align}
Taking summation for \eqref{eq:azuma-bern} with all $i\in[H]$, with probability at least $1-H\delta$, we have
\begin{align}
    \sum_{t=1}^T X_t&= \sum_{i=1}^{H}\sum_{j=0}^{T'}X_{jH+i}\notag\\
    &\leq  \sum_{i=1}^{H}\bigg(\frac{T'}{(1-\gamma)^2}+\sqrt{\frac{2T'\log(1/\delta)}{(1-\gamma)^5}}+ \frac{2\log(1/\delta)}{3(1-\gamma)^3}\bigg)\notag\\
    &\leq\frac{T}{(1-\gamma)^2}+\sqrt{\frac{4T\log(1/\delta)}{(1-\gamma)^6}}+ \frac{4\log(1/\delta)}{3(1-\gamma)^4}\notag\\
    &\leq \frac{2T}{(1-\gamma)^2}+\frac{7\log(1/\delta)}{3(1-\gamma)^4},\label{eq:11}
\end{align}
where the first inequality holds due to \eqref{eq:azuma-bern}, the second inequality holds due to $T'=\lfloor T/H\rfloor +1$ and the third inequality holds due to $x^2+y^2\ge 2xy$.
By the definition of $X_t$, we have
\begin{align}
    \sum_{t=1}^TX_t&=\sum_{t=1}^T\sum_{i=0}^{H}\gamma^{2i+2} \VV^{\pi}_{t+i}(s_{t+i},a_{t+i})\notag\\
    &\ge \sum_{t=1}^T  \VV^{\pi}_{t}(s_{t},a_{t})\sum_{i=0}^{\min{\{H,t-1\}}}\gamma^{2i+2}\notag\\
    &=\sum_{i=0}^{H}\gamma^{2i+2}\sum_{t=1}^T  \VV^{\pi}_{t}(s_{t},a_{t})-\sum_{t=1}^H\VV^{\pi}_{t}(s_{t},a_{t}) \sum_{i=t}^{H}\gamma^{2i+2}\notag\\
    & \geq \frac{\gamma^2 - \gamma^{2H+4}}{1-\gamma^2}\sum_{t=1}^T  \VV^{\pi}_{t}(s_{t},a_{t})-\frac{1}{(1-\gamma)^2}\sum_{t=1}^H \sum_{i=t}^{H}\gamma^{2i+2}, \label{eq:approx}
\end{align}
where the first inequality holds due to $\VV^{\pi}_{t}(s_{t},a_{t})\ge 0$ and  the second inequality holds due to $\VV^{\pi}_{t}(s_{t},a_{t})\leq 1/(1-\gamma)^2$. To further bound \eqref{eq:approx}, we have
\begin{align}
    \frac{\gamma^2 - \gamma^{2H+4}}{1-\gamma^2} = \frac{\gamma^2}{1-\gamma^2}(1-\gamma^{2H+2}) \geq \frac{\gamma^2}{1-\gamma^2}(1-\gamma^{2/(1-\gamma)}) \geq \frac{4\cdot\gamma^2}{5(1-\gamma^2)} \geq \frac{2\gamma^2}{5(1-\gamma)},\label{eq:approx_00}
\end{align}
where the first inequality holds since $2H+2 = 2\lfloor 1/(1-\gamma)\rfloor+2 \geq 2/(1-\gamma)$, the second inequality holds since $0 \leq \gamma^{1/(1-\gamma)} \leq 0.4$ when $0 \leq \gamma  \leq 1$, the last one holds since $1+\gamma \leq 2$. We also have
\begin{align}
    \sum_{t=1}^H \sum_{i=t}^{H}\gamma^{2i+2} \leq \sum_{t=1}^H\frac{\gamma^{2t+2}}{1-\gamma^2} \leq \frac{\gamma^4}{(1-\gamma^2)^2} \leq \frac{\gamma^4}{(1-\gamma)^2}.\label{eq:approx_01}
\end{align}
Substituting \eqref{eq:approx_00} and \eqref{eq:approx_01} into \eqref{eq:approx}, we have
\begin{align}
    \sum_{t=1}^TX_t \geq \frac{2\gamma^2}{5(1-\gamma)}\sum_{t=1}^T  \VV^{\pi}_{t}(s_{t},a_{t}) - \frac{\gamma^4}{(1-\gamma)^4}.\label{eq:approx_xx}
\end{align}
Finally, substituting \eqref{eq:approx_xx} into \eqref{eq:11}, we have
\begin{align}
    \gamma^2\sum_{t=1}^T  \VV^{\pi}_{t}(s_{t},a_{t}) \leq \frac{5T}{1-\gamma} + \frac{35\log(1/\delta)}{6(1-\gamma)^3} + \frac{5\gamma^4}{2(1-\gamma)^3} \leq \frac{5T}{1-\gamma} + \frac{25\log(1/\delta)}{3(1-\gamma)^3}.\notag
\end{align}
Thus, we complete the proof.
\end{proof}

\subsection{Proof of Lemma \ref{lemma:opt-variance-gap}}
\begin{proof}[Proof of Lemma \ref{lemma:opt-variance-gap}]

\noindent On the event $\cE_7$, we have
\begin{align}
   \sum_{i=1}^T(\VV^*(s_i,a_i)-\VV^{\pi}_i(s_i,a_i))&\leq \sum_{i=1}^t \bigg[\PP \big((\vvalue^*)^2-(\vvalue^{\pi}_{i+1})^2\big)\bigg](s_i,a_i)\notag\notag\\
   & =\sum_{i=1}^T\big[\PP (\vvalue^*-\vvalue^{\pi}_{i+1})(\vvalue^*+\vvalue^{\pi}_{i+1})\big](s,a)\notag\\
   &\leq \frac{2}{1-\gamma} \sum_{i=1}^T\bigg[\PP (\vvalue^*-\vvalue^{\pi}_{i+1})\bigg](s_i,a_i)\notag\\
   &\leq \frac{2}{1-\gamma}\sum_{i=1}^T (\vvalue^*(s_{i+1})-\vvalue^{\pi}_{i+1}(s_{i+1}))+\frac{\sqrt{2T\log(1/\delta)}}{(1-\gamma)^2}\notag\\
   &\leq \frac{2}{1-\gamma}\text{Regret}'(T)+\frac{\sqrt{2T\log(1/\delta)}}{1-\gamma}+\frac{2}{(1-\gamma)^2},\notag
\end{align}
where the first inequality holds because of Lemma \ref{lemma:UCB}, the second inequality holds due to $0\leq \vvalue^*(s),\vvalue^{\pi}_{i+1}(s)\leq \frac 1 {1-\gamma}$, the third inequality holds due to the definition of $\cE_7$ and the last inequality holds due to $0\leq \vvalue^*(s)\leq \vvalue_{i}(s)\leq 1/{1-\gamma}$. Thus, we complete the proof.
\end{proof}

\subsection{Proof of Lemma \ref{lemma:estimate-variance-gap}}
\begin{proof}[Proof of Lemma \ref{lemma:estimate-variance-gap}]
\begin{align}
     \sum_{i=1}^T(\VV_{i-1}(s_i,a_i)-\VV^{\pi}_i(s_i,a_i))
     &=\sum_{i=1}^T \EE_{s'\sim \PP_{i-1}(\cdot|s_i,a_i)}[\vvalue_{i-1}^2(s')]-\EE_{s'\sim \PP_{i-1}(\cdot|s_i,a_i)}[\vvalue_{i-1}(s')]^2\notag\\
     &\qquad -\sum_{i=1}^T \EE_{s'\sim \PP(\cdot|s_i,a_i)}[\vvalue^{\pi}_{i+1}(s')^2]-\EE_{s'\sim \PP(\cdot|s_i,a_i)}[\vvalue^{\pi}_{i+1}(s')]^2 \notag\\
     &\leq \underbrace{\sum_{i=1}^T \EE_{s'\sim \PP_{i-1}(\cdot|s_i,a_i)}[\vvalue_{i-1}^2(s')]- \EE_{s'\sim \PP(\cdot|s_i,a_i)}[\vvalue_{i-1}^2(s')]}_{I_1} \notag\\
     &\qquad+  \underbrace{\sum_{i=1}^T \EE_{s'\sim \PP(\cdot|s_i,a_i)}[\vvalue_{i-1}^2(s')]-\EE_{s'\sim \PP(\cdot|s_i,a_i)}[\vvalue^{\pi}_{i+1}(s')^2] }_{I_2}\notag\\
     &\qquad + \underbrace{\sum_{i=1}^T \EE_{s'\sim \PP(\cdot|s_i,a_i)}[\vvalue^*(s')]^2-\EE_{s'\sim \PP_{i-1}(\cdot|s_i,a_i)}[\vvalue^*(s')]^2 }_{I_3},\notag
\end{align}
where the inequality holds due to $\vvalue_{i-1}(s')\ge\vvalue^*(s')\ge \vvalue^{\pi}_{i+1}(s')$.

By the definition of event $\cE_8$, we have
\begin{align}
    \big\|\PP_{i-1}(\cdot|s,a)-\PP(\cdot|s,a)\big\|_1\leq \frac{\sqrt{2S\CC{\log(T/\delta)}}}{\sqrt{N_{i-1}(s,a)\vee 1}}.\label{eq:51}
\end{align}
Thus, for the term $I_1$, since $0\leq \vvalue_{i-1}^2(s')\leq 1/(1-\gamma)^2 $, we have
\begin{align}
    I_1&\leq \sum_{i=1}^T \frac{\sqrt{2S\CC{\log(T/\delta)}}}{(1-\gamma)^2\sqrt{N_{i-1}(s_i,a_i)\vee 1}}\leq \frac{S\sqrt{2AT\CC{\log(3T)}\CC{\log(T/\delta)}}}{(1-\gamma)^2},\label{eq:21}
\end{align}
where the first inequality holds due to \eqref{eq:51} and the second inequality holds due to Lemma \ref{lemma:sum-of-state}. For the term $I_2$, on the event $\cE_6$, we have
\begin{align}
    I_2&\leq \sum_{i=1}^T \bigg[\PP \big((\vvalue_{i-1})^2-(\vvalue^{\pi}_{i+1})^2\big)\bigg](s_i,a_i)\notag\notag\\
   &=\sum_{i=1}^T\big[\PP (\vvalue_{i-1}-\vvalue^{\pi}_{i+1})(\vvalue_{i-1}+\vvalue^{\pi}_{i+1})\big](s,a)\notag\\
   &\leq \frac{2}{1-\gamma} \sum_{i=1}^T\bigg[\PP (\vvalue_{i-1}-\vvalue^{\pi}_{i+1})\bigg](s_i,a_i)\notag\\
   &\leq \frac{2}{1-\gamma}\sum_{i=1}^T (\vvalue_{i-1}(s_{i+1})-\vvalue^{\pi}_{i+1}(s_{i+1}))+\frac{\sqrt{2T\CC{\log(2/\delta)}}}{1-\gamma}\notag\\
    &\leq \frac{4S}{1-\gamma}+\frac{2}{1-\gamma}\sum_{i=1}^T (\vvalue_{i+1}(s_{i+1})-\vvalue^{\pi}_{i+1}(s_{i+1}))+\frac{\sqrt{2\CC{\log(T/\delta)}}}{(1-\gamma)^2}\notag\\
   &\leq \frac{2}{1-\gamma}\text{Regret}'(T)+\frac{\sqrt{2T\CC{\log(1/\delta)}}}{(1-\gamma)^2}+\frac{4S+2}{(1-\gamma)^2},\label{eq:22}
\end{align}
where the first inequality holds due to $\vvalue_{i-1}(s')\ge \vvalue^*(s')\ge \vvalue^{\pi}_{i+1}(s')$, the second inequality holds due to $0\leq \vvalue_{i-1}(s'),\vvalue^{\pi}_{i+1}(s') \leq 1/(1-\gamma)$, the third inequality holds due to the definition of  event $\cE_6$ and the forth inequality holds due to
$ \vvalue_{i-1}(s')\ge \vvalue_{i+1}(s')$.

For the term $I_3$, since $0\leq \vvalue^*(s')^2\leq 1/(1-\gamma)^2,$ on the event $\cE_8$, we have
\begin{align}
   I_3 &\leq \sum_{i=1}^T \frac{\sqrt{2S\log(T/\delta)}}{(1-\gamma)^2\sqrt{N_{i-1}(s_i,a_i)}}\leq \frac{S\sqrt{2AT\log(T/\delta)\log(3T)}}{(1-\gamma)^2},\label{eq:23}
\end{align}
where the first inequality holds due to \eqref{eq:51} and the second inequality holds due to Lemma \ref{lemma:sum-of-state}. Taking an union bound for \eqref{eq:21}, \eqref{eq:22} and \eqref{eq:23}, with probability at least $1-3\delta$, we have
\begin{align}
    \sum_{i=1}^t(\VV_{i-1}(s_i,a_i)-\VV^{\pi}_i(s_i,a_i))\leq \frac{2\text{Regret}'(T)}{1-\gamma}+\frac{9S\sqrt{2AT\log(T/\delta)\log(3T)}}{(1-\gamma)^2}.\notag
\end{align}

\end{proof}


\subsection{Proof of Lemma \ref{lemma:regret-state}}
\begin{proof}[Proof of Lemma \ref{lemma:regret-state}]

\noindent For each $i\in [H]$,$s\in \cS$ and $t\in [T]$, if $N_t(s)=0$, the we have 
\begin{align}
    \text{Regret}'(t,s,h)&=0\leq \frac{16SAU^2\sqrt{N_t(s)}}{(1-\gamma)^{2.5}}+\frac{20S^2A^{1.5}U^{4.5}}{(1-\gamma)^{3.5}}. \notag
\end{align}
Otherwise, we have
\begin{align}
\text{Regret}'(t,s,h)&= \sum_{1\leq i\leq t,s_i=s} \gamma^h\big[\vvalue_{i+h}(s_{i+h}) - \vvalue^{\pi}_{i+h}(s_{i+h})\big]\notag\\
    &= \sum_{1\leq i\leq t,s_i=s} \gamma^h \big[\qvalue_{t}(s_{i+h},a_{i+h}) - \vvalue^{\pi}_{i+h}(s_{i+h})\big]\notag\\
    &\leq \sum_{1\leq i\leq t,s_i=s} \gamma^{h+1}[\PP_{i+h-1}\vvalue_{i+h-1}](s_{i+h},a_{i+h})+\gamma^h \text{UCB}_{i+h-1}(s_{i+h},a_{i+h})\notag \\
    &\qquad -\gamma^{h+1} \PP\vvalue^{\pi}_{i+h+1}(s_{i+h},a_{i+h})\notag\\
    &=I_1 + I_2 + I_3 + \gamma^hI_4 + \text{Regret}'(t,s,h+1),  
    \label{eq:tele}
\end{align}
where the first inequality holds due to definition update rule  \eqref{eq:update}. $I_1,\dots, I_4$ are defined as follows.
\begin{align}
    I_1 &= \sum_{1\leq i\leq t,s_i=s}\gamma^{h+1}(\vvalue_{i+h-1}(s_{i+h+1})-\vvalue_{i+h+1}(s_{i+h+1})),\notag \\
    I_2 &= \sum_{1\leq i\leq t,s_i=s}\gamma^{h+1}[(\PP_{i+h-1}-\PP)\vvalue_{i+h-1}](s_{i+h},a_{i+h})\notag, \\
    I_3& = \sum_{1\leq i\leq t,s_i=s}\gamma^{h+1}\big[\PP(\vvalue_{i+h-1}-\vvalue^{\pi}_{i+h+1})\big](s_{i+h},a_{i+h}),\notag \\
    &\qquad -\gamma^{h+1}\big[\vvalue_{i+h-1}(s_{i+h+1})-\vvalue^{\pi}_{i+h+1}(s_{i+h+1})\big],\notag \\
    I_4& = \sum_{1\leq i\leq t,s_i=s} \text{UCB}_{i+h-1}(s_{i+h},a_{i+h}).\notag
\end{align}
For the term $I_1$, we have
\begin{align}
    \sum_{1\leq i\leq t,s_i=s}\gamma^{h+1}(\vvalue_{i+h-1}(s_{i+h+1})-\vvalue_{i+h+1}(s_{i+h+1}))&\leq \sum_{i=1}^t\sum_{s'\in \cS}\vvalue_{i+h-1}(s')-\vvalue_{i+h+1}(s')\notag\\
    &\leq \frac{2S}{1-\gamma},\label{eq:I_1-state}
\end{align}
where the first inequality holds due to $\vvalue_{i+h-1}(s')\ge \vvalue_{i+h+1}(s') $ and the second inequality holds due to $0\leq \vvalue_t(s)\leq 1/(1-\gamma)$.

For the term $I_2$, with probability at least $1-\delta$, we have
\begin{align}
    &\sum_{1\leq i\leq t,s_i=s}\gamma^{h+1}[(\PP_{i+h-1}-\PP)\vvalue_{i+h-1}](s_{i+h},a_{i+h})\notag\notag \\
    &\leq \sum_{1\leq i\leq t,s_i=s}\frac{\gamma^{h+1}\sqrt{2SU}}{(1-\gamma)\sqrt{N_{i+h-1}(s_{i+h},a_{i+h})\vee 1}}\notag\\
    &\leq \frac{\gamma^{h+1}\sqrt{2SU}}{(1-\gamma)}\sqrt{N_t(s)\sum_{1\leq i\leq t,s_i=s} \frac{1}{N_{i+h-1}(s_{i+h},a_{i+h})\vee 1}}\notag\\
    &\leq \frac{\sqrt{2SU}}{1-\gamma} \sqrt{N_t(s)SAU}\notag\\
    &=\frac{SU\sqrt{2N_t(s)A}}{1-\gamma},\label{eq:I_2-state}
\end{align}
where the first inequality holds due to Lemma \ref{lemma:azuma} and the definition of $U$, the second inequality holds due to Cauchy-Schwarz inequality and the third inequality holds due to Lemma \ref{lemma:sum-of-state}.

For the term $I_3$, Since the random process $s_{i+h+1}\sim \PP(\cdot| s_{i+h},a_{i+h})$ is dependent with whether $s_{i+1},..,s_{i+h+1}=s$, we cannot directly use Lemma \ref{lemma:azuma} to bound this term. However, we can use the same technique in the proof of Lemme \ref{lemma:total-variance}, which divide the time horizon into $H$ sub-horizon and use Lemma \ref{lemma:azuma} for each sub-horizon. Compared with the upper bound of $I_3$ in proof of Theorem \ref{thm:2}, this technique will lead to a gap of $\sqrt{H}$ and we have
\begin{align}
    &\sum_{ i\leq t,s_i=s}\gamma^{h+1}\big[\PP(\vvalue_{i+h-1}-\vvalue^{\pi}_{i+h+1})\big](s_{i+h},a_{i+h})-\gamma^{h+1}\big[\vvalue_{i+h-1}(s_{i+h+1})-\vvalue^{\pi}_{i+h+1}(s_{i+h+1})\big]\notag\\
    &\qquad\leq \frac{\sqrt{2N_t(s)U}}{(1-\gamma)} \sqrt{H}\notag\\
    &\qquad\leq \frac{2U\sqrt{N_t(s)}}{(1-\gamma)^{1.5}},\label{eq:I_3-state}
\end{align}
where the second inequality holds due to the definition of $U$.
For the term $I_4$, we have
\begin{align}
    &\sum_{1\leq i\leq t,s_i=s} \text{UCB}_{i+h-1}(s_{i+h},a_{i+h})\notag\\
    &\qquad\leq  \underbrace{\sum_{1\leq i\leq t,s_i=s} 
    \sqrt{\frac{8U\VV_{i+h-1}(s_{i+h},a_{i+h})}{N_{i+h-1}(s_{i+h},a_{i+h})\vee 1}}}_{I_{41}}+\underbrace{\sum_{1\leq i\leq t,s_i=s}\frac{8U}{(1-\gamma)\big(N_{i+h-1}(s_{i+h},a_{i+h})\vee 1\big)}}_{I_{42}}\notag\\
        &\qquad +\underbrace{\sum_{1\leq i\leq t,s_i=s}\sqrt{\frac{8\sum_{s'}\PP_{i+h}(s'|s_{i+h},a_{i+h})\min\big\{100B_{i+h}(s'),{1}{/(1-\gamma)^2}\big\}}{N_{i+h-1}(s_{i+h},a_{i+h})\vee 1}}}_{I_{43}}.\label{eq:I_4-state}
\end{align}
For the term $I_{41}$, with probability at least $1-\delta$, we have
\begin{align}
    &\sum_{1\leq i\leq t,s_i=s} 
    \sqrt{\frac{8U\VV_{i+h-1}(s_{i+h},a_{i+h})}{N_{i+h-1}(s_{i+h},a_{i+h})\vee 1}}\notag\\
    &\qquad\leq \sqrt{8U} \sqrt{\sum_{1\leq i\leq t,s_i=s} \VV_{i+h-1}(s_{i+h},a_{i+h})}\sqrt{\sum_{1\leq i\leq t,s_i=s} \frac{1}{N_{i+h-1}(s_{i+h},a_{i+h})\vee 1} }\notag\\
    &\qquad\leq U\sqrt{8SA} \sqrt{\sum_{1\leq i\leq t,s_i=s} \VV_{i+h-1}(s_{i+h},a_{i+h})}\notag\\
    &\qquad\leq U\sqrt{8SA}\sqrt{\frac{2N_t(s)}{(1-\gamma)^2}},
\end{align}
where the first inequality holds due to Cauchy-Schwarz inequality, the second inequality holds due to Lemma \ref{lemma:sum-of-state}, the last inequality holds due to $0\leq \VV_{i+h-1}(s_{i+h},a_{i+h})\leq {1}/{(1-\gamma)^2}$.

For the term $I_{42}$, by Lemma \ref{lemma:sum-of-state}, we have
\begin{align}
   \sum_{1\leq i\leq t,s_i=s}\frac{8U}{(1-\gamma)\big(N_{i+h-1}(s_{i+h},a_{i+h})\vee 1\big)}\leq \frac{8 SAU^2}{1-\gamma}.\label{eq:13}
\end{align}
For the term $I_{43}$, with probability at least $1-2\delta$, we have
\begin{align}
    &\sum_{1\leq i\leq t,s_i=s}\sqrt{\frac{8\sum_{s'}\PP_{i+h}(s'|s_{i+h},a_{i+h})\min\big\{100B_{i+h}(s'),{1}{/(1-\gamma)^2}\big\}}{N_{i+h-1}(s_{i+h},a_{i+h})\vee 1}}\notag\\
    &\leq \sqrt{8\sum_{1\leq i\leq t,s_i=s}\frac{1}{N_{i+h-1}(s_{i+h},a_{i+h})\vee 1}}\notag \\
    &\qquad \cdot \sqrt{\sum_{1\leq i\leq t,s_i=s} \sum_{s'}\PP_{i+h}(s'|s_{i+h},a_{i+h})\min\bigg\{100B_{i+h}(s'),\frac{1}{(1-\gamma)^2}\bigg\}}\notag\\
    &\leq \sqrt{8SAU}\sqrt{\sum_{1\leq i\leq t,s_i=s} \sum_{s'}\PP_{i+h}(s'|s_{i+h},a_{i+h})\min\bigg\{100B_{i+h}(s'),\frac{1}{(1-\gamma)^2}\bigg\}}\notag\\
    &\leq  \sqrt{8SAU}\bigg[\sum_{1\leq i\leq t,s_i=s}\bigg(\frac{\sqrt{SU}}{(1-\gamma)^2\sqrt{N_{i+h-1}(s_{i+h},a_{i+h})\vee 1}}\notag \\
    &\qquad +\sum_{s'}\PP(s'|s,a)\min\bigg\{100B_{i+h}(s'),\frac{1}{(1-\gamma)^2}\bigg\}\bigg)\bigg]^{1/2}\notag\\
    &\leq \sqrt{8SAU}\bigg[\frac{SU\sqrt{AN_t(s)}}{(1-\gamma)^2}+\frac{\sqrt{2N_t(s)U}}{(1-\gamma)^2}\notag \\
    &\qquad +\sum_{1\leq i\leq t,s_i=s} \min\bigg\{\frac{100S^2A^2U^5}{(1-\gamma)^5\big(N_{i+h-1}(s_{i+h+1})\vee 1 \big)},\frac{1}{(1-\gamma)^2}\bigg\}\bigg]^{1/2}\notag\\
    &\leq \sqrt{8SAU} \sqrt{\frac{SU\sqrt{AN_t(s)}}{(1-\gamma)^2}+\frac{\sqrt{2N_t(s)U}}{(1-\gamma)^2}+\frac{100S^3A^2U^6}{(1-\gamma)^5}},
\end{align}
where the first inequality holds due to Cauchy-Schwarz inequality, the second inequality holds due to Lemma \ref{lemma:sum-of-state}, the third inequality holds due to Lemma \ref{lemma:azuma}, the forth inequality holds due to Lemma \ref{lemma:azuma} and the last inequality holds due to Lemma \ref{lemma:sum-of-state}.
Substituting \eqref{eq:I_1-state}, \eqref{eq:I_2-state}, \eqref{eq:I_3-state}, \eqref{eq:I_4-state} into \eqref{eq:tele}, with probability at least $1-4H\delta$, we have
\begin{align}
    \text{Regret}'(t,s,h)&\leq \text{Regret}'(t,s,h+1)+ \frac{16 SAU\sqrt{N_t(s)}}{(1-\gamma)^{1.5}}+\frac{20S^2A^{1.5}U^{3.5}}{(1-\gamma)^{2.5}}.\label{eq:tele1}
\end{align}
Notice that
\begin{align}
    \text{Regret}'(t,s,H)&= \sum_{1\leq i\leq t,s_i=s} \gamma^H\big[\vvalue_{i+H}(s_{i+H}) - \vvalue^{\pi}_{i+H}(s_{i+H})\big]\notag\\
    &\leq \sum_{1\leq i\leq t,s_i=s} \frac{\gamma^H}{1-\gamma}\notag\\
    &\leq \sum_{1\leq i\leq t,s_i=s} \frac{1}{T}\notag\\
    &\leq 1,\notag
\end{align}
where the first inequality holds due to $\vvalue_{i+H}(s_{i+H}) - \vvalue^{\pi}_{i+H}(s_{i+H})\leq 1/{(1-\gamma)}$ and the second inequality holds due to definition of $H$. Thus, taking summation of \eqref{eq:tele1} with all $h\in [H]$, with probability at least $1-H^2\delta$, we have 
\begin{align}
    \text{Regret}'(t,s,0)\leq \frac{16SAU^2\sqrt{N_t(s)}}{(1-\gamma)^{2.5}}+\frac{20S^2A^{1.5}U^{4.5}}{(1-\gamma)^{3.5}}.\label{eq:14}
\end{align}
In addition, if $N_t(s)>0$, we have
\begin{align}
    \vvalue_t(s)-\vvalue^*(s)&\leq \frac{1}{N_t(s)}\sum_{1\leq i\leq t,s_i=s}  \vvalue_i(s)-\vvalue^*(s)\notag\\
    &\leq \frac{1}{N_t(s)}\sum_{1\leq i\leq t,s_i=s}  [\vvalue_i(s)-\vvalue^{\pi}_i(s)]\notag\\
    &\leq \frac{16SAU^2}{(1-\gamma)^{2.5}\sqrt{N_t(s)}}+\frac{20S^2A^{1.5}U^{4.5}}{(1-\gamma)^{3.5}N_t(s)},\notag
\end{align}
where the first inequality holds due to $V_i(s)$ is decreasing, the second inequality holds due to $\vvalue^*(s)\ge \vvalue^{\pi}_i(s)$ and the third inequality holds due to \eqref{eq:14}.
Notice that when $N_t(s)\ge {S^2AU^3}/{(1-\gamma)^2}$, we have
\begin{align}
     \vvalue_t(s)-\vvalue^*(s)&\leq \frac{16SAU^2}{(1-\gamma)^{2.5}\sqrt{N_t(s)}}+\frac{20S^2A^{1.5}U^{4.5}}{(1-\gamma)^{3.5}N_t(s)} \leq \frac{36SAU^2}{(1-\gamma)^{2.5}\sqrt{N_t(s)}}.\notag
\end{align}
Otherwise, we have
\begin{align}
    \vvalue_t(s)-\vvalue^*(s)&\leq \frac{1}{1-\gamma}\leq \frac{36SAU^2}{(1-\gamma)^{2.5}\sqrt{N_t(s)}}.\notag
\end{align}
Thus, we complete the proof of Lemma \ref{lemma:regret-state}.
\end{proof}

\end{document}